\documentclass{siamonline250211}
\def\@#1{\bm{#1}}                               
\def\!#1{\mathbf{#1}}                           

\newcommand{\diff}{\:\mathrm{d}}                

\newcommand{\cond}{\:|\:}                       

\newcommand{\tcont}{\:\overline{\times}\:}      
\newcommand{\tct}{\times\cdots\times}           

\newcommand{\N}{\mathbb{N}}
\newcommand{\R}{\mathbb{R}}

\newcommand{\cA}{\mathcal{A}}
\newcommand{\cB}{\mathcal{B}}
\newcommand{\cC}{\mathcal{C}}

\newcommand{\cL}{\mathcal{L}}
\newcommand{\cM}{\mathcal{M}}
\newcommand{\cO}{\mathcal{O}}

\newcommand{\cT}{\mathcal{T}}
\newcommand{\cW}{\mathcal{W}}
\newcommand{\cX}{\mathcal{X}}
\newcommand{\cY}{\mathcal{Y}}

\newcommand{\flt}{\text{flat}}                  

\usepackage{algorithmic}

\usepackage{amsfonts}
\usepackage{amssymb}
\usepackage{bbm}
\usepackage{bm}
\usepackage{booktabs}
\usepackage{cancel}
\usepackage{graphicx}
\usepackage{xcolor}
\usepackage{mdframed}
\usepackage{multirow}
\usepackage{nameref}
\usepackage{standalone}
\usepackage{stmaryrd}
\usepackage{tikz}

\newtheorem{remark}[theorem]{Remark}

\DeclareMathOperator*{\argmin}{argmin}

\title{Tensor-Train Weak SINDy: Identifying High-Dimensional Nonlinear Dynamics}
\author{Will Houser\thanks{Department of Applied Mathematics, University of Colorado, Boulder, CO 80309-0526 USA (will.houser@colorado.edu, vadu4217@colorado.edu, dmbortz@colorado.edu) }\and Vanja Dukic\footnotemark[1]\and David M. Bortz\footnotemark[1]}
                                                                                   
\begin{document}
\maketitle
\begin{abstract}
    Weak Sparse Identification of Nonlinear Dynamics (WSINDy) provides a noise-robust approach for learning dynamical systems from data without requiring numerical differentiation. However, for high-dimensional systems, tensor-product libraries of candidate functions grow exponentially with the state dimension, making standard WSINDy expensive in both computation and memory. The Multidimensional Approximation of Nonlinear Dynamics (MANDy) addresses this scaling through a tensor-train (TT) representation of the candidate library, but does not provide a mechanism for sparse model selection. Here, we combine these approaches to develop TT-WSINDy, which performs the weak-form transformation, regression, and sparsification in TT format. We show that the TT formulation recovers the corresponding WSINDy regression problem and derive polynomial time and memory complexity bounds for the tensor-train sparsification procedure. Numerical experiments demonstrate robustness to measurement noise and computational savings for high-dimensional systems.
\end{abstract}

\section{Introduction\label{sec:Introduction}}
Data-driven identification of dynamical systems can broadly be divided into output-error (OE) and equation-error (EE) approaches. Output-error methods compare observed data with trajectories obtained by solving a candidate dynamical system. While widely used, this approach can be computationally expensive, subject to numerical solver error, and highly sensitive to numerical and optimization choices, including the discretization method and resolution, optimization algorithm, and associated hyperparameters \cite{NardiniBortz2019InverseProblems}. Equation-error methods instead evaluate the governing equations directly on the data, avoiding repeated forward simulation. One such approach is \textit{Sparse Identification of Nonlinear Dynamics (SINDy)} \cite{BruntonProctorKutz2016ProcNatlAcadSci}, which identifies a parsimonious model by selecting a sparse subset from a prescribed library of candidate functions. The development of SINDy has generated substantial interest in EE-based model discovery, and has been extended in the contexts of PDEs \cite{RudyBruntonProctorEtAl2017SciAdv}, control inputs \cite{BruntonProctorKutz2016IFAC-PapersOnLine}, and implicit or rational dynamics \cite{KahemanKutzBrunton2020ProcRSocA}. The framework has seen applications across various scientific domains, including fluid dynamics \cite{LoiseauBrunton2018JFluidMech}, systems biology \cite{ManganBruntonProctorEtAl2016IEEETransMolBiolMulti-ScaleCommun}, and chemical kinetics \cite{HoffmannFrohnerNoe2019JChemPhys}.

Two difficulties with SINDy-based methods are the computation of derivative data and the sensitivity to measurement noise. Often, only the state data $\!X$ are observed, while the derivative data $\dot{\!X}$ are necessary for computation of the residual. The need to approximate $\dot{\!X}$ can render the overall learning algorithm sensitive to noisy data. Building off of SINDy, several strategies have been developed to address computation of the derivative and accommodate imprecise data \cite{vanbreugelNumericalDifferentiationNoisy2020, WentzDoostan2023ComputerMethodsinAppliedMechanicsandEngineering,  kahemanAutomaticDifferentiationSimultaneously2022, FaselKutzBruntonEtAl2022ProcRSocMathPhysEngSci}.

\textit{Weak-form} methods circumvent the derivative estimation problem by integrating both sides of the equation against a smooth, compactly-supported test function, and applying integration by parts to transfer the derivative from the measured state to the test function. The integration also smooths the raw data, making the resulting regression problem more robust to noise. The use of weak formulations for parameter estimation dates to at least Shinbrot \cite{Shinbrot1954NACATN3288}, and their increasing adoption for data-driven model discovery reflects an emerging consensus that weak-form approaches are highly effective at addressing the challenges of measurement noise and derivative approximation \cite{GurevichGoldenReinboldEtAl2024JFluidMech,Messenger.Bortz2021JournalofComputationalPhysics,Messenger.Bortz2021MultiscaleModelSimul,PantazisTsamardinos2019Bioinformatics,SchaefferMcCalla2017Phys.Rev.E,wangVariationalSystemIdentification2020}. Of these approaches, we build on \textit{Weak SINDy (WSINDy)} \cite{Messenger.Bortz2021JournalofComputationalPhysics,Messenger.Bortz2021MultiscaleModelSimul}, from which we adopt both notation and sparse regression techniques. WSINDy and related weak-form methods have subsequently been applied to ODEs \cite{Bortz.etal2023BullMathBiol, Chawla.etal2025HandbookofVisualExperimentalandComputationalMathematicsBridgesthroughData, Heitzman-Breen.etal2026BullMathBiol, messengerCoarseGrainingHamiltonianSystems2023a, RummelMessengerBeckerEtAl2026SIAMJSciComput, Lopez.etal2026arXiv260423269, Tian.etal2026arXiv260530432}, PDEs \cite{LyonsDukicBortz2025PLoSComputBiol, MessengerDallAneseBortz2022ProcThirdMathSciMachLearnConf, MinorMessengerDukicEtAl2025JournalofGeophysicalResearchMachineLearningandComputation, Vasey.etal2025JournalofComputationalPhysics}, SDEs \cite{2024HandbookofNumericalAnalysis, MinorElderdBortzEtAl2026SIAMJLifeSci}, reduced order modeling \cite{HeTranBortzEtAl2025IntJNumerMethodsEng, TranHeMessengerEtAl2024ComputMethodsApplMechEng, WangMessenger2025PhysRevA}, and closure discovery \cite{MessengerSouthworthHammerEtAl2025arXiv251011840}. For an overview  of weak-form methods, see \cite{Messenger.etal2024SIAMNewsa}.

\textit{Tensor decompositions} have also seen increased interest in recent decades. The concept of a tensor decomposition dates back to Hitchcock's 1927 work on the canonical polyadic (CP) decomposition \cite{Hitchcock1927JournalofMathematicsandPhysics}, rediscovered in 1970 by Carroll and Chang as CANDECOMP \cite{CarrollChang1970Psychometrika} and Harshman as PARAFAC \cite{Harshman}. Subsequent to these works and the 1960s work of Tucker \cite{Tuck1963a}, each of which originated in the psychometrics literature, tensor methods increasingly spread to applied mathematics, including signal processing, numerical linear algebra, and computer vision, and were synthesized in the foundational review by Kolda and Bader \cite{KoldaBader2009SIAMRevb}. 

Classical tensor formats can suffer from the curse of dimensionality, with storage or computational costs that scale poorly with tensor order. Finding the CP rank of a general tensor is NP-hard \cite{Hastad1990JAlgorithms}, while the storage cost of the Tucker decomposition grows exponentially with the number of modes. This motivated the development of tensor formats whose storage scales linearly with tensor order for bounded ranks, most notably the tree-structured hierarchical Tucker format \cite{Hackbusch2014ActaNumerica}. A special case is the tensor-train (TT) decomposition, introduced in 2011 by Oseledets \cite{Oseledets2011SIAMJSciComput}, and known in computational physics community since the early 1990s as the matrix product state (MPS) \cite{Fannes.etal1992CommunicationsinMathematicalPhysics}. The TT decomposition is the tensor representation used throughout this work.

Tensor decompositions have shown recent promise in improving existing data-driven methods, providing the ability to work in a high-dimensional space without forming the corresponding high-dimensional tensor in memory. For example, \cite{zhang2026robustmomentbasedestimationspectral} decomposes high-dimensional moment tensors to form a polynomial-time variant of the generalized method of moments (GMM) procedure. Klus, Gel\ss, and others showed that the tensor-train format can be applied to dynamic mode decomposition (DMD) \cite{Klusetal2018Nonlinearity} and, through the Multidimensional Approximation of Nonlinear Dynamics (MANDy) method \cite{GelssKlusEisertEtAl2019JComputNonlinearDyn}, to the identification of nonlinear dynamical systems. MANDy addresses the high-dimensional scaling problem, but does not incorporate the weak form or sparse model selection, motivating its combination with WSINDy in the present work.

Our work proceeds as follows. \Cref{sec:notation+preliminaries} establishes notation, gives exposition on SINDy and WSINDy, and defines the tensor-train format and relevant tensor operations. \Cref{sec:WSTINDy} presents the TT-WSINDy problem construction and the TT-MSTLS algorithm to identify a sparse solution. \Cref{sec:WSINDyT=WSINDy} proves the connection between the TT-WSINDy problem construction and the original WSINDy problem. \Cref{sec:TTwsindySpeed} establishes the time and memory complexity of TT-MSTLS and shows that it circumvents the curse of dimensionality, and introduces an algorithm that improves computational efficiency for large datasets. \Cref{sec:results} provides results demonstrating the effectiveness of TT-WSINDy, both in robustness to noise and computational efficiency.

\section{Notation and preliminaries\label{sec:notation+preliminaries}}
In this section, we describe the problem setting and the \textit{Weak SINDy (WSINDy)} algorithm used to identify the approximate dynamics. We then introduce the \textit{tensor-train (TT) format}, which admits a fast pseudoinverse computation. The TT format is compatible with both the weak formulation and sparse regression, providing the foundation for the TT-WSINDy algorithm introduced in Section \ref{sec:WSTINDy}.

\subsection{Problem setting}
We consider the problem of learning a first-order ordinary differential equation in $D$ dimensions, taking the form
\begin{equation}
    \dot{\!x}(t) = \!F(\!x(t)), \quad \!x(0) = \!x_0 \in \R^D, \quad 0 \leq t \leq T
    \label{eqn:ODEproblem}
\end{equation}
We have measurement data $\!X \in \R^{D \times M}$, sampled at timepoints $(t_1,\dots,t_M)$ and given by
\begin{equation}
    \!X[d,m] = \!x_d(t_m) + \epsilon_{d,m}, \quad d \in [D], m \in [M]
\end{equation}
where $\{\epsilon_{d,m}\}_{d\in[D],m\in[M]}$ is independent and identically distributed measurement noise from some common distribution. In this work, we assume that $\!F$ is \textit{linear in parameters}, i.e. given dimensionwise by
\begin{equation}
    \!F_d(\!x(t)) = \sum_{j=1}^J \!W[d,j]\cdot g_j(\!x(t))
    \label{eqn:linearInParameters}
\end{equation}
where the family of functions $\@g := \{g_j\}_{j\in[J]}$, with $g_j: \R^D \rightarrow \R$, is known. In this context, discovering the dynamics in \cref{eqn:ODEproblem} is equivalent to discovering the \textit{coefficient matrix} $\!W \in \R^{D\times J}$. 

The SINDy algorithm \cite{BruntonProctorKutz2016ProcNatlAcadSci} addresses problems of this form by constructing a \textit{feature matrix} $\Theta(\!X) \in \R^{J \times  M}$ whose rows are discretizations of each $g_j \in \@g$:
\begin{equation}
    \Theta(\!X)[j,m] = g_j(\!X[:,m])
\end{equation}
Given pointwise derivative data $\dot{\!X}$, evaluated on the same spatiotemporal grid as $\!X$, we have the discrete problem
\begin{equation}
    \dot{\!X} \approx \!W\Theta(\!X)
    \label{eqn:strongFormDiscrete}
\end{equation}
Then, SINDy finds $\widehat{\!W}$ such that
\begin{equation}
\begin{split}
    & \widehat{\!W} \text{ is sparse, and} \\
    & \widehat{\!W} \approx \argmin_{\!W} \|\dot{\!X} - \!W\Theta(\!X)\|_F
    \label{eqn:SINDyOpitmizationProblem}
\end{split}
\end{equation}
For the original work on this problem formulation and computation of $\widehat{\!W}$, see \cite{BruntonProctorKutz2016ProcNatlAcadSci}.

In many physical applications, $\dot{\!X}$ is not observed and must be estimated from $\!X$, with the resulting loss of accuracy compounded by measurement noise. The WSINDy method addresses these issues. 

\subsection{WSINDy\label{sec:WSINDy}}
Several groups have independently discovered that using the \textit{weak form} of \cref{eqn:ODEproblem} bypasses the need to compute $\dot{\!X}$ and makes the problem robust to noise \cite{GurevichGoldenReinboldEtAl2024JFluidMech,Messenger.Bortz2021JournalofComputationalPhysics,Messenger.Bortz2021MultiscaleModelSimul,PantazisTsamardinos2019Bioinformatics,SchaefferMcCalla2017Phys.Rev.E,wangVariationalSystemIdentification2020}.
Here, we briefly review the \textit{Weak SINDy (WSINDy)} formulation used in this work.

We introduce a smooth \textit{test function} $\varphi : \R \rightarrow \R$, compactly supported in some interval $[a,b] \subset [0,T]$. Integrating $\varphi$ against both sides of \cref{eqn:ODEproblem} and applying integration by parts to the LHS results in the weak formulation
\begin{equation}
    -\int_0^T \dot{\varphi}(t)\!x(t) \diff t = \int_0^T \varphi(t)\!F(\!x(t))dt
    \label{eqn:weakFormCts}
\end{equation}
Now, let $\@\varphi \in \R^B$ be the discretization of $\varphi$ on $\!t \cap (a,b)$. Then, approximating the integrals in \cref{eqn:weakFormCts} via the trapezoidal rule yields the discrete problem
\begin{equation}
    -\!X\dot{\@\Phi} \approx \!W\Theta(\!X)\@\Phi
    \label{eqn:weakFormDiscrete}
\end{equation}
where $\@\Phi$ is the banded matrix
\begin{equation}
    \@\Phi = \begin{bmatrix}
        \@\varphi[1] & & & \\
        \vdots & \@\varphi[1] \\
        \@\varphi[B] & \vdots \\
        & \@\varphi[B] & \ddots & \@\varphi[1] \\
        &&& \vdots \\
        &&& \@\varphi[B]
    \end{bmatrix} \in \R^{M \times (M - B + 1)}
    \label{eqn:Phi}
\end{equation}
and $\dot{\@\Phi}$ is the same, but with test function derivative $\dot{\@\varphi}$. Crucially, \cref{eqn:weakFormDiscrete} has no dependence on $\dot{\!X}$, thus avoiding   potentially unstable numerical differentiation. Moreover, the integration in \cref{eqn:weakFormCts} mollifies the noise in $\!X$. Analogous to \cref{eqn:SINDyOpitmizationProblem}, the WSINDy optimization problem is then to find a sparse $\widehat{\!W}$ such that
\begin{equation}
    \widehat{\!W} \approx \argmin_{\!W} \|\!X\dot{\@\Phi} + \!W\Theta(\!X)\@\Phi\|_F
    \label{eqn:WSINDyOptimizationProblem}
\end{equation}
For details on choosing $\varphi$, see \cite{Tran.Bortz2026SIAMJSciComput}. Section \ref{sec:WSTINDy} recasts this weak-form construction into a tensor format. 

\subsection{Tensor preliminaries}
In this work, a tensor is a multidimensional array, generalizing a matrix. As we specify a matrix $\!X \in \R^{m \times n}$ by the dimensions $(m,n)$, we specify the shape of an order-$N$ tensor $\cX \in \R^{I_1\tct I_N}$ by the tuple $(I_1,\dots,I_N)$, called the \textit{modes} of $\cX$. The storage cost of such a tensor is $\cO(I^N)$, where $I := \max_{n\in[N]} I_n$ is the maximum mode size. For large $N$, storage of such an object may easily be infeasible, and the time cost of performing even basic algebraic operations can explode. This is referred to as the \textit{curse of dimensionality}. However, it is often the case that immense amounts of memory and time can be saved by representing a high-dimensional tensor through a low-rank decomposition. Such decompositions exploit a low-rank correlation structure. We will show that the WSINDy learning problem admits such a representation.

In order to discuss relevant tensor decompositions in detail, we introduce two binary tensor operations.
\begin{definition}
    Let $\cX \in \R^{I_1 \times \cdots\times I_N}$, $\cY \in \R^{J_1\times\cdots\times J_M}$. Their \textit{outer product}, denoted $\cX \circ \cY \in \R^{I_1\times\cdots\times I_N \times J_1\times\cdots\times J_M}$, is defined element-wise as
    \[(\cX \circ \cY)[i_1,\dots,i_N,j_1,\dots,j_M] = \cX[i_1,\dots,i_N]\cY[j_1,\dots,j_M]\]
\end{definition}
This definition generalizes the typical vector outer product $\!u\!v^T$. Similarly, tensor contraction generalizes matrix multiplication, collapsing adjacent tensor dimensions.
\begin{definition}
    Let $\cX \in \R^{I_1 \times \cdots \times I_N \times J_1 \times \cdots \times J_L}$, $\cY \in \R^{J_1 \times \cdots \times J_L \times K_1 \times \cdots \times K_M}$. Their \textit{tensor contraction}, denoted $\cX \tcont \cY \in \R^{I_1 \times\cdots\times I_N\times K_1\times\cdots\times K_M}$, is defined element-wise as 
    \[(\cX\tcont\cY)[i_1,\dots,i_N,k_1,\dots,k_M] = \sum_{j_1,\dots,j_L}\cX[i_1,\dots,i_N,j_1,\dots,j_L]\cY[j_1,\dots,j_L,k_1,\dots,k_M]\]
    \cite{BallardKolda}
\end{definition}
Note that there is ambiguity in the notation  $\tcont$, as it does not specify which indices are to be contracted. In this work, this will be clear from context.

We will also require representations of tensors as matrices, called \textit{matricizations}. To specify the ordering used in this operation, we let $b_N$ denote a bijection that maps a multi-index to a single index
\begin{align*}
    b_N : [I_1] \times [I_2] \tct [I_N] &\rightarrow [I_1I_2\cdots I_N] \\
    (i_1,i_2,\dots,i_N) &\mapsto \overline{i_1,i_2,\dots,i_N}
\end{align*}
e.g., lexicographical order. In implementation, different programming languages may use different orderings (column-major in Matlab vs row-major in NumPy, for example). The specific ordering does not matter in our context; it is only important that all operations that invoke an order are defined with respect to the same underlying bijection.

\begin{definition}
    Let $\cX \in \R^{I_1 \times \cdots \times I_N}$, and $n \in [N]$. We denote the \textit{mode-$n$ unfolding}, or \textit{mode-$n$ matricization} of $\cX$ by
    \[\cX_{(n)} = \cX\big|_{I_1,\dots,I_n}^{I_{n+1},\dots,I_N} \in \R^{I_1\cdots I_n \times I_{n+1}\cdots I_N}\]
    defined element-wise as
    \[\left(\cX_{(n)}\right)\bigg[\overline{i_1,\dots,i_n},\overline{i_{n+1},\dots,i_N}\bigg] = \cX[i_1,\dots,i_N]\]
\end{definition}
In this work, we will use the following special cases of unfoldings.
\begin{definition}
    Let $\cX \in \R^{I_1 \times \cdots \times I_N}$. We define the \textit{right unfolding} and \textit{left unfolding} of $\cX$, respectively, as
    \[\!R(\cX) = \cX_{(1)}, \quad \!L(\cX) = \cX_{(N-1)}\]
\end{definition}

\subsection{The tensor-train format}
Central to this work is the \textit{tensor-train (TT)} format. A special case of the \textit{hierarchical Tucker format} \cite{Hackbusch2014ActaNumerica}, the TT format decomposes an arbitrary tensor into a sequence of order-$3$ \textit{TT cores}. Each core shares its first mode length with the previous core, and the last with the next, hence, the ``train'' of tensors. This structure allows for efficient computation of a pseudoinverse without forming the full tensor.

\begin{definition}
    Let $\cT \in \R^{I_1\tct I_N}$. We say $\cT$ is in the \textit{tensor-train format} if it admits the expression
    \begin{equation}
        \cT = \sum_{r_0=1}^{R_0}\sum_{r_1=1}^{R_1}\cdots \sum_{r_N=1}^{R_N} \cT^{(1)}[r_0,:,r_1] \circ \cT^{(2)}[r_1,:,r_2] \circ \cdots \circ \cT^{(N)}[r_{N-1},:,r_N]
        \label{eqn:TensorTrain}
    \end{equation}
    over \textit{core tensors}
    \begin{equation*}
        \cT^{(n)} \in \R^{R_{n-1}\times I_n\times R_n}, \quad n \in [N]
    \end{equation*}
    We let $R_0 = R_N = 1$, and we say that the tuple $(R_1,\dots,R_{N-1})$ gives the \textit{TT ranks}\footnote{Note that there are multiple definitions of tensor rank -- TT rank is distinct from the commonly-used \textit{CP rank} \cite{KoldaBader2009SIAMRevb}.} of $\cT$. \cite{Oseledets2011SIAMJSciComput}
\end{definition}

\begin{figure}
    \centering
    \includegraphics[width=0.8\linewidth]{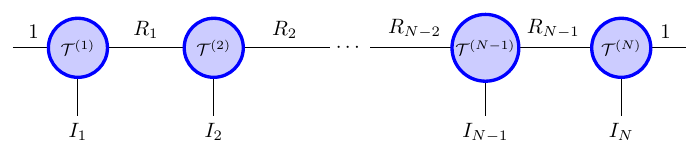}
    \caption{Diagram of tensor-train $\cT = \big\llbracket \cT^{(1)},\dots,\cT^{(N)}\big\rrbracket \in \R^{I_1\tct I_N}$ with ranks $(R_1,\dots,R_{N-1})$. Adjacent cores share a mode. The boundary ranks $R_0, R_N$ are fixed to $1$, so the outer cores can be considered matrices.}
    \label{fig:ttnodeedgediagram}
\end{figure}
We write TT cores compactly as two-dimensional arrays with vector elements
\begin{equation}
    \cT^{(n)} = \left\llbracket \begin{matrix}
        \cT^{(n)}[1,:,1] & \cdots & \cT^{(n)}[1,:,R_n] \\
        \vdots & \ddots & \vdots \\
        \cT^{(n)}[R_{n-1},:,1] & \cdots & \cT^{(n)}[R_{n-1},:,R_n]
    \end{matrix}\right\rrbracket
\end{equation}
and use the following shorthand to define a TT in terms of its cores
\begin{equation}
    \cT = \text{TT}\big\llbracket \cT^{(1)},\dots,\cT^{(N)} \big\rrbracket
\end{equation}
An essential property of the TT format is that it admits an efficiently-computable, analytic pseudoinverse. Given $\cT$ as in \cref{eqn:TensorTrain}, its pseudoinverse\footnote{
Since this object does not have boundary ranks of 1, it is not in TT format, strictly speaking. We can consider $\cT^\dag$ a \textit{cyclic tensor train}, or a sum of $R_{N-1}$ tensor trains, scaled by $\sigma_{r_{N-1}}^{-1}$ \cite{GelssKlusEisertEtAl2019JComputNonlinearDyn}.
} is given by
\begin{equation}
    \cT^\dag := \sum_{r_1=1}^{R_1}\cdots \sum_{r_{N-1}=1}^{R_{N-1}}\sigma_{r_{N-1}}^{-1}\widetilde{\cT}^{(N)}[r_{N-1},:,1]\circ \widetilde{\cT}^{(1)}[1,:,r_1] \circ \cdots \circ \widetilde{\cT}^{(N-1)}[r_{N-2},:,r_{N-1}]
    \label{eqn:TTPseudoinverse}
\end{equation}
with $\cT^\dag \in \R^{I_N \times I_1 \tct I_{N-1}}$. $\widetilde{\cT}^{(1)},\dots,\widetilde{\cT}^{(N-1)}$ are left-orthonormalized counterparts of the cores $\cT^{(1)},\dots,\cT^{(N-1)}$, and $\widetilde{\cT}^{(N)}$ is $\cT^{(N)}$, right-orthonormalized. Algorithms to compute these can be found in \cite{Oseledets2011SIAMJSciComput, Klusetal2018Nonlinearity}, and additional discussion of orthonormalization of TT cores can be found in \cite{GelssKlusEisertEtAl2019JComputNonlinearDyn}. The coefficients $\{\sigma_{r_{N-1}}\}_{r_{N-1} \in [R_{N-1}]}$ are obtained through a sequence of SVDs of the original cores $\{\cT^{(n)}\}_{n\in[N]}$, so the pseudoinverse can be computed by working only with the TT cores. We never form the full tensor and thus never leave the low-rank latent space we have constructed. This results in the following lemma.
\begin{lemma}
    Given tensor train $\cT \in \R^{I_1\tct I_N \times K}$, there exists an algorithm that computes $\cT^\dag$ with time complexity $\cO(NIR^3 + KR^2)$, where $I$ is the maximum of the first $N$ modes, and $R$ is the maximum TT rank. \cite{GelssKlusEisertEtAl2019JComputNonlinearDyn}
    \label{lem:TTPI-runtime}
\end{lemma}
For a detailed description and proof, see \cite{Klusetal2018Nonlinearity}; a brief overview of the algorithm is provided in \cref{sec:TT-PI}. The obtained pseudoinverse is compatible with the equivalent matrix problem:\footnote{In fact, the pseudoinverse is defined such that it is compatible with our choice of matricization. In this work, we matricize based on the left-unfolding, but we can just as easily define a different matricization that will induce a different pseudoinverse. Recent work has shown that the choice of matricization can affect the conditioning of the underlying problem \cite{Popaetal2021Geophysics}.} it holds that
\begin{equation}
    \!R(\cT^\dag) = \!L(\cT)^\dag
\end{equation}
In contrast, direct computation of $\!L(\cT)^\dag$ using a matrix pseudoinverse can be estimated as $\cO(I^NK^2)$. Thus, provided the TT ranks remain sufficiently small, \cref{lem:TTPI-runtime} can yield significant computational savings. Section \ref{sec:WSTINDy} recasts the WSINDy learning problem into a TT pseudoinverse computation that exploits this advantage.

\renewcommand{\arraystretch}{1.5}
\begin{table}[h]
\caption{General tensor notation and relevant unary operations}
\footnotesize
  \centering
  \begin{tabular}{|c|c|c|}
    \hline
    Symbol & Domain & Meaning  \\\hline 
    $x,c,C$ (plain font) & $\R$ & scalar \\\hline 
    $\!x$ (bold lowercase) & $\R^N$ & vector \\\hline
    $\!X$ (bold uppercase) & $\R^{M\times N}$ & matrix \\\hline 
    $\cX, \@\Psi$ (calligraphy, bold greek) & $\R^{I_1\times\cdots\times I_N}$ & tensor \\\hline
    $(i_1,\dots,i_N)$ & $\N^{I_1\times\cdots\times I_N}$ & multi-index \\\hline
    $b_N$ & $C([I_1]\tct[I_N], [I_1\cdots I_N])$ & bijective ordering of multi-indices \\\hline
    $\cX_{(n)}, \cX\big|^{I_{n+1},\dots,I_N}_{I_1,\dots,I_n}$ & $\R^{I_1\cdots I_n \times I_{n+1}\cdots I_N}$ & mode-$n$ unfolding of $\cX$ \\\hline
    $\!R(\cX)$ & $\R^{I_1 \times I_2\cdots I_N}$ & right unfolding of $\cX$ \\\hline
    $\!L(\cX)$ & $\R^{I_1\cdots I_{N-1} \times I_N}$ & left unfolding of $\cX$ \\\hline
     $\text{TT}\llbracket \cT^{(1)},\dots,\cT^{(N)}\rrbracket$ & $\R^{I_1\tct I_N}$ & Tensor train decomposition \\\hline 
  \end{tabular}
  \label{tab:tensorNotation}
\end{table}

\section{TT-WSINDy\label{sec:WSTINDy}}
We now introduce the TT-WSINDy construction and algorithm. Following the \textit{Multidimensional Approximation of Nonlinear Dynamics (MANDy)} method \cite{GelssKlusEisertEtAl2019JComputNonlinearDyn}, we recast the ``flat'' WSINDy learning problem in  tensor-train format, allowing a large candidate library to be represented without explicitly forming the corresponding feature matrix. We then  combine \textit{modified sequential thresholding least squares} (MSTLS) with TT regression to obtain a sparse model while remaining in the low-rank tensor representation during the coarse regression stage. The weak-formulation provides the noise robustness of WSINDy without sacrificing the tensor structure.

In this work, we consider $\@f := \{f_j\}_{j\in[J]}$ to be a family of \textit{basis functions}, from which we derive the full candidate function library
\begin{equation}
    \@g = \{g : \R^D \rightarrow \R \cond g(\!x) = f_{j_1}(\!x_1)f_{j_2}(\!x_2)\cdots f_{j_D}(\!x_D), (j_1,\dots,j_D) \in [J]^D  \}
    \label{eqn:gFunctionLib}
\end{equation}
where each $f_j : \R \rightarrow \R$. Because $|\@g| = J^D$, the feature matrix $\Theta(\!X)$ defined in \cref{eqn:weakFormDiscrete} suffers from the curse of dimensionality. For large $D$, operations involving $\Theta(\!X)$ may be computationally infeasible. This necessitates an alternate formulation of the WSINDy optimization problem.
\subsection{Problem construction\label{sec:WSINDyTconstruction}}
Given basis functions $\@f = \{f_j\}_{j \in [J]}$ and data points $\!X = \{x_{d,m}\}_{d\in [D], m \in [M]}$ sampled over time points $\!t = \{t_m\}_{m\in[M]}$, define the \textit{feature vector}
\begin{equation}
    \@\theta^{(d)}(t_m) = \begin{bmatrix}
        f_1(x_{d,m}) \\ \vdots \\ f_J(x_{d,m})
    \end{bmatrix} \in \R^J
    \label{eqn:dimmajorfeaturevector}
\end{equation}
containing the evaluations of each basis function $f_j$ at $x_{d,m}$.
Then, define the set of \textit{feature cores}, written as matrices with vector elements, as
\begin{equation}
\begin{split}
    \@\Theta^{(1)} &= \left\llbracket \begin{matrix}
        \@\theta^{(1)}(t_1) & \cdots & \@\theta^{(1)}(t_M) 
    \end{matrix} \right\rrbracket \in \R^{J \times M} \\
    \@\Theta^{(d)} &= \left\llbracket \begin{matrix}
        \@\theta^{(d)}(t_1) & & 0 \\
        & \ddots \\
        0 & & \@\theta^{(d)}(t_M)
    \end{matrix}\right\rrbracket \in \R^{M \times J \times M}; \quad d = 2,\dots,D
\end{split}
\end{equation}
and \textit{weak core}\footnote{Note that we use the same notation as \cref{eqn:Phi}, as both definitions give the same matrix.} as
\begin{equation}
    \@\Phi = \left\llbracket\begin{matrix}
        \@e_1 \star \@\varphi \\ \vdots \\ \@e_M\star \@\varphi
    \end{matrix}\right\rrbracket \in \R^{M \times M'}
\end{equation}
where $\{e_m\}_{m\in[M]}$ is the Euclidean basis of $\R^M$, $\@\varphi$ is the discretization of a test function $\varphi$ on $\!t \cap (a,b)$, $M' = M - B + 1$, and $\star$ denotes \textit{discrete cross-correlation} (see \cref{def:tensorcrosscorrelation}). Over these, we define the \textit{weak feature tensor} as the rank-$(M,\dots,M)$ tensor train
\begin{equation}
    \@\Theta(\!X, \@\varphi) := \text{TT}\left\llbracket \@\Theta^{(1)}, \dots, \@\Theta^{(D)}, \@\Phi \right\rrbracket \in \R^{\overbrace{J \times \cdots \times J}^D \times M'}
\end{equation}
Under this construction, the optimization problem becomes finding a \textit{coefficient tensor} $\widehat{\cW} \in \R^{D \times J\times\cdots\times J}$ such that
\begin{equation}
\begin{split}
    &\widehat{\cW} \text{ is sparse, and} \\
    &\widehat{\cW} \approx \argmin_{\cW} \left\|\!X \star \dot{\@\varphi} + \cW \:\overline{\times}\: \@\Theta(\!X, \@\varphi) \right\|_F 
    \label{eqn:WINSy-ToptimizationProblem}
\end{split}
\end{equation}
Note that we have used the same construction of the feature cores as in \cite{GelssKlusEisertEtAl2019JComputNonlinearDyn}. The differences are the cross-correlation $\!X \star \dot{\@\varphi}$ of the data $\!X$ and the folding of the test function $\varphi$ into the $(D+1)^\text{th}$ core. We show in \cref{sec:WSINDyT=WSINDy} that this allows us to solve the same weak-form problem as in \cref{sec:WSINDy} while preserving the desirable tensor-train structure of the original formulation.

\begin{figure}[htb]
    \centering
    \includegraphics[width=0.8\linewidth]{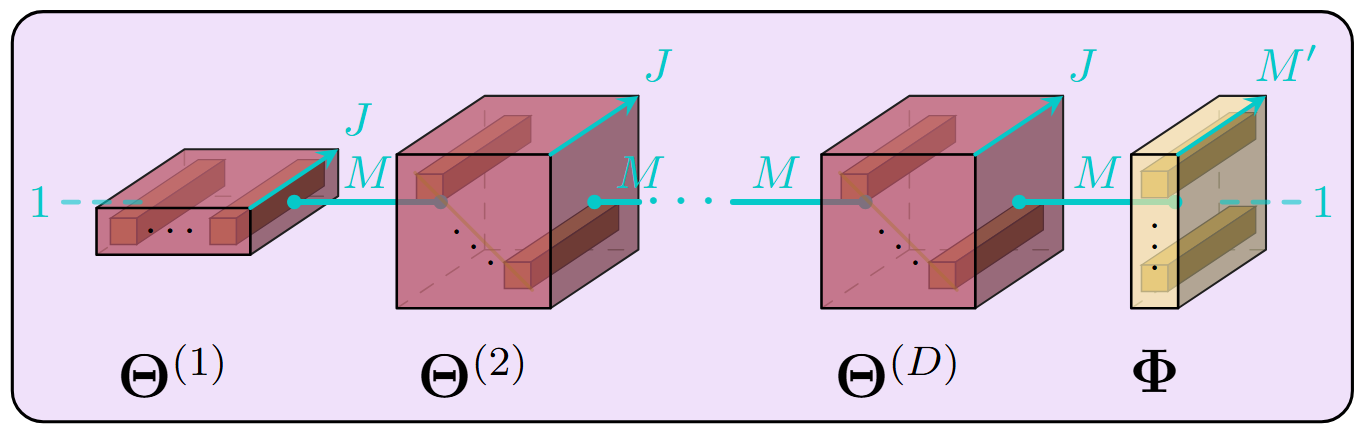}
    \caption{Schematic of the weak feature tensor.}
    \label{fig:weakfeaturetensor}
\end{figure}

Also note that we may alternatively define a feature vector as evaluating a single $f_j$ over every dimension $d \in [D]$. Depending on  $J$ and $D$, this may be preferable to the construction we have shown here. We detail this construction in \cref{sec:functionMajor}.

\subsection{Tensor-train regression}
To solve \cref{eqn:WINSy-ToptimizationProblem}, we introduce the \textit{tensor-train modified sequential thresholding least squares} (TT-MSTLS) algorithm, which adapts the known STLS sparse solve to the tensor train regression problem. First, suppose that we have support
\begin{equation}
    S = \{S^{(1)}, \dots, S^{(D)}\} \subset [J]^D
\end{equation}
Then, we construct the \textit{mask tensor} $\cM(S) \in \R^{J\tct J\times M'}$, a Kruskal tensor \cite{BallardKolda} defined over \textit{mask vectors} $\!m^{(d)} \in \R^J$, as
\begin{equation}
\begin{split}
    &\cM(S) := \left[\bigcirc_{d=1}^{D} \!m^{(d)}\right] \circ \@1_{M'} \\
    &\text{s.t. } \!m^{(d)}[j] = \begin{cases}
        1 & j \in S^{(d)} \\
        0 & \text{otherwise}
    \end{cases} 
    \label{eqn:maskTensor}
\end{split}
\end{equation}
where $\@1_{M'}$ denotes the length-$M'$ vector of ones. To obtain $S_\ell$ from a coefficient tensor $\cW \in \R^{J\tct J}$, we define the \textit{coarse support} of $\cW$ as
\begin{equation}
\begin{split}
    &\text{coarse-supp}(\cW, \lambda) = \big\{S^{(1)},\dots,S^{(D)}\big\} \subset [J]^D \\
    &\text{s.t. } j \in S^{(d)} \iff \left\|\cW\bigg[\overbrace{:\cdots :}^{d-1},j,\overbrace{:\cdots:}^{D-d}\bigg]\right\|_F^2 \ge \lambda 
\end{split}
\end{equation}
In other words, we define support of $\cW$ slicewise. With these constructions, we define the \textit{tensor-train sequential thresholding least squares} (TT-STLS) in \cref{alg:TT-STLS}.

\begin{algorithm}[htb]
\caption{TT-STLS}\label{alg:TT-STLS}
\begin{algorithmic}[1]
    \REQUIRE \quad Weak feature tensor $\@\Theta(\!X,\@\varphi) \in \R^{J\tct J \times M'}$ \\
        \quad\quad\quad LHS vector $\!y \in \R^{M'}$ \\
        \quad\quad\quad Thresholding parameter $\lambda \in \R$
    \ENSURE Coefficient estimate $\cW^\lambda \in \R^{J\tct J}$ \\
        \quad\quad\quad coarse support $S^\lambda \subset [J]^D$
    \STATE Set initial values $\@\Psi_0 := \@\Theta(\!X,\@\varphi)$, $S_0 := [J]^D$
    \FOR{$\ell = 1,\dots,JD$}
        \STATE Compute pseudoinverse $(\@\Psi_{\ell-1})^\dag$
        \STATE Compute weights $\cW_{\ell} \leftarrow \!y \tcont (\@\Psi_{\ell-1})^\dag$
        \STATE Compute support $S_{\ell} \leftarrow \text{coarse-supp}(\cW_\ell,\lambda)$
        \IF {$S_\ell = S_{\ell-1}$}
            \STATE \textbf{break}
        \ENDIF
        \STATE Iterate $\@\Psi_\ell \leftarrow \@\Psi_{\ell-1} \odot \cM(S_\ell)$
    \ENDFOR
    \RETURN $\cW^\lambda := \cW_{\ell},\ S^\lambda := S_{\ell}$
\end{algorithmic}
\end{algorithm}

We now construct the \textit{tensor-train modified  sequential thresholding least squares (TT-MSTLS)} algorithm. This tracks closely with the matrix-based MSTLS algorithm, introduced in \cite{Messenger.Bortz2021JournalofComputationalPhysics}, with adaptations to accommodate the TT structure. As with MSTLS, we perform a line search over a set of thresholding parameters $\@\lambda$ with respect to an auxiliary loss function $\cL$, to find the parameter $\widehat{\lambda}$ which induces the lowest loss. We take the loss function to be
\begin{equation}
    \cL(\lambda) = \frac{\|(\cW^\lambda - \cW^0)\tcont \@\Theta(\!X,\@\varphi)\|_2}{\|\cW^0 \tcont \@\Theta(\!X,\@\varphi)\|_2} + \frac{\|\cW^\lambda\|_0}{J^D}
\end{equation}
where $\cW^\lambda := \text{TT-STLS}(\@\Theta, \!y, \lambda)$ is the coefficient estimate induced by $\lambda$, and $\cW^0 := \!y \tcont \@\Theta^\dag$ is the initial unsparsified TT pseudoinverse solution. The first added term penalizes distance from $\cW^0 \tcont \@\Theta(\!X, \@\varphi)$, and the second encourages sparsity of the discovered coefficient tensor. This is parallel to the matrix MSTLS$(\!A,\!b,\lambda)$ loss function, introduced in \cite{Messenger.Bortz2021JournalofComputationalPhysics}, to solve the linear problem $\!A\!w = \!b$.
\begin{equation}
    \cL_\flt(\lambda) = \frac{\|\!A(\!w^\lambda - \!w^0)\|_2}{\|\!A\!w^0\|_2} + \frac{\|\!w^\lambda\|_0}{\#(\!A)}
\end{equation}
where, similarly, $\!w^\lambda := \text{STLS}(\!A, \!b, \lambda)$ and $\!w^0 := \!A^\dag\!b$. We define TT-MSTLS in \cref{alg:TT-MSTLS}.

\begin{algorithm}[htb]
\caption{\label{alg:TT-MSTLS} TT-MSTLS}
\begin{algorithmic}[1]
    \REQUIRE \quad Weak feature tensor $\@\Theta(\!X,\@\varphi) \in \R^{J\tct J \times M'}$\\
        \quad\quad\quad LHS vector $\!y \in \R^{M'}$ \\
        \quad\quad\quad Set of thresholding parameters $\@\lambda \subset \R$ \\
        \quad\quad\quad Auxiliary loss function $\cL : \R^+ \rightarrow \R^+$
    \ENSURE Coarse support $\widehat{S} \subset [J]^D$
    \FOR{$\lambda \in \@\lambda$}
        \STATE Compute $\cW^\lambda, S^\lambda \leftarrow \text{TT-STLS}(\@\Theta(\!X, \@\varphi),\!y;\lambda)$
        \STATE Compute $\cL(\lambda)$
    \ENDFOR
    \STATE Set $\widehat{\lambda} = \argmin_{\lambda \in \@\lambda} \cL(\lambda)$ \\
    \RETURN $\widehat{S} := S^{\widehat{\lambda}}$
\end{algorithmic}
\end{algorithm} 

\subsection{TT-WSINDy algorithm}
We define TT-WSINDy in \cref{alg:TT-WSINDy} as the two-stage application of TT-MSTLS and MSTLS. TT-WSINDy calls the coarse sparsification procedure TT-STLS to reduce the search space of candidate functions in the TT-induced latent space, then calls the standard matrix MSTLS algorithm. We give a visualization of this procedure in \cref{fig:ttwsindy}. \\
\begin{algorithm}[htb]
\caption{\label{alg:TT-WSINDy} TT-WSINDy}
\begin{algorithmic}[1]
    \REQUIRE \quad Data matrix $\!X \in \R^{D\times M}$, sampled on time points $\@t \subset \R$ \\
        \quad\quad\quad Family of basis functions $\@f = \{f_j\}_{j\in[J]}$ \\
        \quad\quad\quad Test function and derivative $\varphi, \dot\varphi \in C_0((a,b),\R)$ \\
        \quad\quad\quad TT-MSTLS loss function $\cL : \R^+ \rightarrow \R^+$ \\
        \quad\quad\quad MSTLS loss function $\cL_\flt: \R^+ \rightarrow \R^+$ \\
        \quad\quad\quad Sets of thresholding parameters $\@\lambda, \@\lambda_\flt \subset \R$
    \ENSURE Coefficient estimate $\widehat{\!W} \in \R^{D \times J \tct J}$
    \STATE Compute $\@\varphi, \dot{\@\varphi}, $ by discretizing $\varphi, \dot{\varphi}$ on $(a,b) \cap \!t$
    \STATE Construct $\@\Theta(\!X,\@\varphi) \in \R^{J\tct J \times M'}$ from $\!X,\@f,\@\varphi$
    \STATE Compute $\!X \star \dot{\@\varphi} \in \R^{D \times M'}$ \\
    \FOR{$d = 1,\dots, D$}
        \STATE Set $\!y_d \leftarrow (-\!X \star \dot{\@\varphi})[d,:]$
        \STATE Compute coarse support $\widehat{S} \leftarrow \text{TT-MSTLS}(\@\Theta(\!X,\@\varphi), \!y_d; \cL, \@\lambda)$
        \STATE Compute coefficients $\widehat{\!w}_d \leftarrow \text{MSTLS}\big(\!L\big(\@\Theta(\!X,\@\varphi) \odot \cM(\widehat{S})\big), \!y_d; \cL_\flt, \@\lambda_\flt\big)$
    \ENDFOR
    \RETURN $\widehat{\!W}$
\end{algorithmic}
\end{algorithm}
If we omit TT-MSTLS and set $\widehat{S} \leftarrow [J]^D$, TT-WSINDy collapses to WSINDy. We will see that the inclusion of TT-MSTLS can yield substantial memory and time complexity improvements.

\begin{figure}
    \centering
    \includegraphics[width=\linewidth]{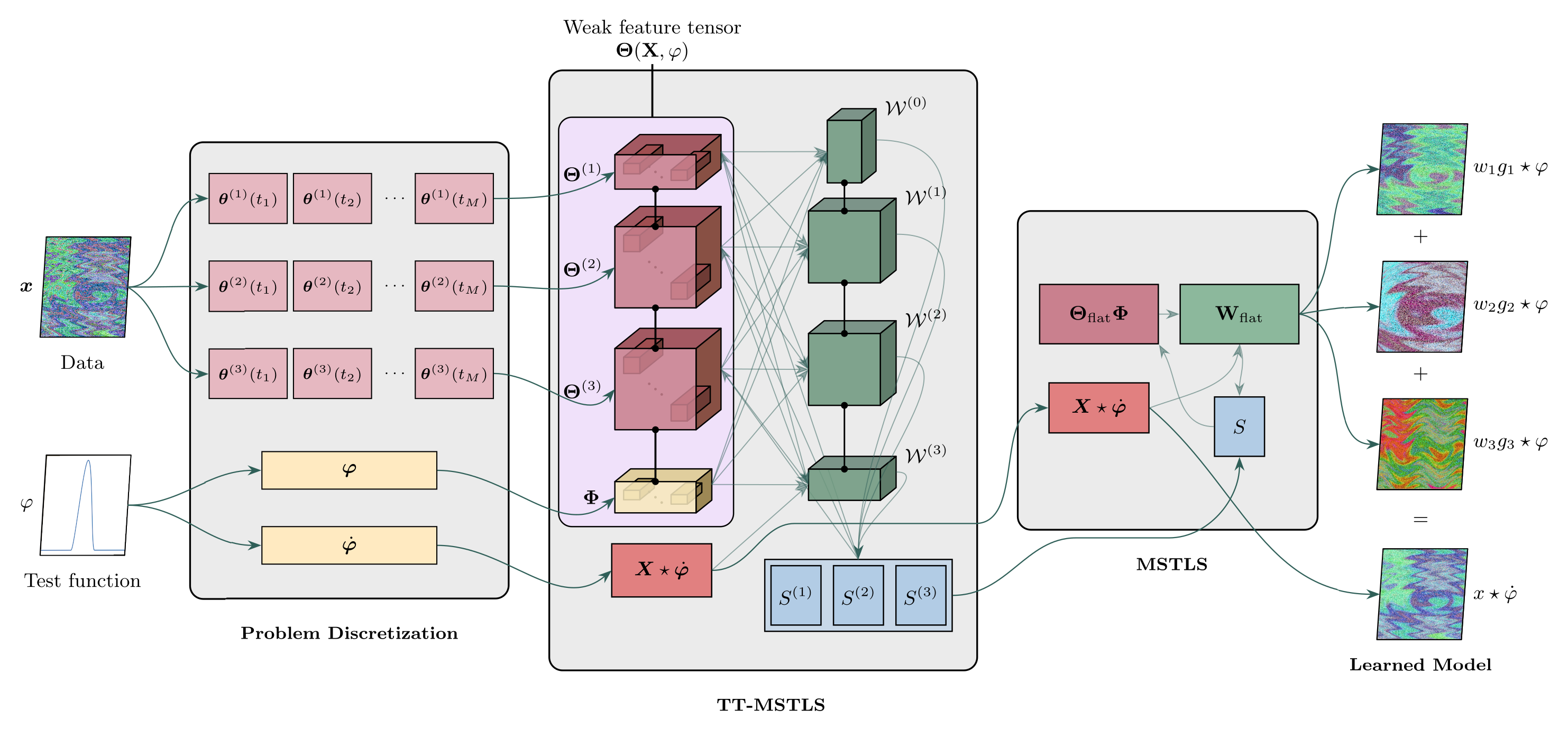}
    \caption{Schematic of the full TT-WSINDy algorithm.}
    \label{fig:ttwsindy}
\end{figure}
                                                                                  
\section{TT-WSINDy agrees with WSINDy\label{sec:WSINDyT=WSINDy}}
In this section, we prove \cref{prop:WSINDyT=WSINDy}, which gives that TT-WSINDy solves the same sparse regression problem as WSINDy, and that the construction is the weak form of the strong form tensor problem.

Define the \textit{strong feature tensor} as $\@\Theta(\!X) := \@\Theta(\!X, [1])$, i.e. replacing the final core with the identity matrix $\!I \in \R^{M \times M}$. And define
\begin{equation}
    \!L(\@\Theta(\!X)) = \@\Theta(\!X)\big|_{J,\dots,J}^M =: \@\Theta_\flt(\!X) \in \R^{J^D \times M}
    \label{eqn:ThetaFlat}
\end{equation}
which is the SINDy feature matrix for the function library $\@g$, defined in \cref{eqn:gFunctionLib}. Now, we establish some technical lemmata. We first give the general definition of discrete cross-correlation. 
\begin{definition}\label{def:tensorcrosscorrelation}
    Let $\cA \in \R^{I_1\tct I_N}$, $\cB \in \R^{J_1 \tct J_N}$. Their \textit{discrete cross-correlation} is denoted $\cA \star \cB \in \R^{H_1 \tct H_N}$, where $H_n = I_n - J_n + 1$ for $n \in [N]$, and is defined element-wise as
    \begin{equation*}
        (\cA \star \cB)[h_1,\dots,h_N] = \sum_{j_1=1}^{J_1} \cdots \sum_{j_N=1}^{J_N} \cA[j_1 + h_1 - 1, \dots, j_N + h_N - 1]\cB[j_1,\dots,j_N]
    \end{equation*}
\end{definition}
We see that the weak-form matrix multiplication is equivalent to taking a discrete cross-correlation of the data with the test function discretization. 
\begin{lemma}
    Take $\@\varphi, \dot{\@\varphi}$ and $\@\Phi, \dot{\@\Phi} \in \R^{M \times M'}$ as above, and let $\!A \in \R^{H \times M}$. Then $\!A\dot{\@\Phi} = \!A \star \dot{\@\varphi}$ and $\!A\@\Phi = \!A \star \@\varphi$.
    \label{lem:MatrixcorrEqualsCorr}
\end{lemma}
\begin{proof} 
We may express $\@\Phi$ entrywise as
\begin{equation}
    \@\Phi[m,m'] = \begin{cases}
        \@\varphi[k] & m = (m' - 1) + k; \; k \in [B] \\
        0 & \text{otherwise}
    \end{cases}
\end{equation}
So we have, using the definitions of matrix multiplication and discrete cross-correlation\footnote{Note that, for compatibility with \cref{def:tensorcrosscorrelation}, we consider $\@\varphi, \dot{\@\varphi} \in \R^{1\times B}$ in this computation. This specifies that $\@\varphi$ is in the time axis only; we do not correlate across the first dimension.},
\begin{align*}
    (\!A\@\Phi)[h,m'] &= \sum_{m=1}^M \!A[h,m]\@\Phi[m,m'] = \sum_{k=1}^B \!A[h, m' + k - 1]\@\varphi[k] = (\!A \star \@\varphi)[h,m']
\end{align*}
and identically for $\dot{\@\varphi}, \dot{\@\Phi}$. 
\end{proof}

We now relate the flat weak-form feature matrix to the weak feature tensor.
\begin{lemma}
    Let $\@\varphi$ and $ \@\Theta(\!X,\@\varphi)$ be defined as in \cref{sec:WSINDyTconstruction}. Then,
    $
        \@\Theta_\flt(\!X) \star \@\varphi = \!L(\@\Theta(\!X,\@\varphi))
    $.\label{lem:RightUnfoldingofFeatureTensor}
\end{lemma}

\begin{proof}
Using the element-wise formula for $\@\Theta(\!X,\@\varphi)$, we have
\begin{equation*}
\!L(\@\Theta(\!X,\@\varphi))_{\overline{j_1,\dots,j_D},m'} = \sum_{m=1}^M \left[\prod_{d=1}^D \@\theta^{(d)}(t_m)[j_d]\right](\@e_m \star \@\varphi)[m'] = \sum_{m=1}^M\left[  \prod_{d=1}^D f_{j_d}(x_{d,m})\right]\@\varphi[m - m' + 1]
\end{equation*}
as $(\@e_m \star \@\varphi)[m'] = \sum_{h=1}^M\@\varphi[h + m' - 1]\@e_m[h] = \@\varphi[m - m' + 1]$. And we have
\begin{equation*}
    \prod_{d=1}^D f_{j_d}(x_{d,m}) = \@\Theta_\flt(\!X)\bigg[\overline{j_1,\dots,j_D}, m\bigg]
\end{equation*}
Therefore we have the desired equality by the definition of discrete cross-correlation.
\begin{equation*}
(\!L(\@\Theta(\!X,\@\varphi)))_{\overline{j_1,\dots,j_D},m'} = \sum_{m=1}^M \@\Theta_\flt(\!X)\bigg[\overline{j_1,\dots,j_D}, m\bigg]\@\varphi[m - m' + 1] = (\@\Theta_\flt(\!X) \star \@\varphi)_{\overline{j_1,\dots,j_D},m'}
\end{equation*}
\text{ }
\end{proof}

\begin{lemma}
    Let $\cA \in \R^{A \times B_1 \times \cdots \times B_N}$ and $\cC \in \R^{B_1 \times \cdots \times B_N \times C}$. Then, $\!R(\cA)\!L(\cC) = \cA \tcont \cC$.
    \label{lem:unfoldingproductequalscontraction}
\end{lemma}
\begin{proof}
We write $\cA \tcont \cC$ elementwise and reindex the sum using the bijection $b_N$:
\[(\cA \tcont \cC)[a,c] = \sum_{b_1=1}^{B_1}\cdots \sum_{b_N=1}^{B_N} \cA[a,b_1,\dots,b_N]\cC[b_1,\dots,b_N,c] = \sum_{b=1}^{B_1\cdots B_N}\!R(\cA)[a,b]\!L(\cC)[b,c]\]
which is equal to $\big(\!R(\cA)\!L(\cC)\big)[a,c]$.
\end{proof}
\begin{corollary}
    Define $\cW_\flt$ to be the solution to the SINDy matrix problem, as
    \[\!R(\cW) = \cW \big|_D^{J,\dots,J} =: \cW_\flt \in \R^{D \times J^D}\]
    Then,
    $
        \cW \tcont \@\Theta(\!X) = \cW_\flt \@\Theta_\flt(\!X)
    $.
    \label{rmk:MANDy=FlatStrongFormProblem}
\end{corollary}
\begin{proof}
By definition, $\cW_\flt = \!R(\cW)$ and $\@\Theta_\flt(\!X) = \!L(\@\Theta(\!X))$. Thus the result follows immediately from \cref{lem:unfoldingproductequalscontraction}. 
\end{proof}
In other words, $\cW$ solves the strong-form tensor problem $\dot{\!X} = \cW \tcont \@\Theta(\!X)$ if and only if $\cW_\flt$ solves the SINDy problem $\dot{\!X} = \cW_\flt\@\Theta_\flt(\!X)$. We may now show the main technical result of this section, described below.

\begin{proposition}
    Let $\!X, \@\varphi$, and $\dot{\@\varphi}$ be defined as in \cref{sec:WSINDy}, $\@\Theta_\flt(\!X)$ as in \cref{eqn:ThetaFlat}, and $\@\Theta(\!X,\@\varphi)$ as in \cref{sec:WSINDyTconstruction}. We have that if $\cW$ solves the TT-WSINDy optimization problem, then $\!R(\cW)$ solves the equivalent WSINDy optimization problem. I.e.,
    \begin{equation}
        \!R\left( \argmin_{\cW} \|\!X\dot{\@\Phi} + \cW \tcont \@\Theta(\!X,\@\varphi)\|_F \right) = \argmin_{\!W} \|\!X\dot{\@\Phi} + \!W \@\Theta_\flt(\!X)\@\Phi\|_F
        \label{eqn:TT-WSINDy=WSINDy}
    \end{equation}
    From which it follows that the construction in \cref{sec:WSINDyTconstruction} does indeed describe the weak form of the MANDy construction given in \cite{GelssKlusEisertEtAl2019JComputNonlinearDyn}.
    \label{prop:WSINDyT=WSINDy}
\end{proposition}
\begin{proof}
Using \cref{lem:MatrixcorrEqualsCorr,lem:RightUnfoldingofFeatureTensor,lem:unfoldingproductequalscontraction}, we have
\begin{align*}
    (\cW_\flt\@\Theta_\flt(\!X))\@\Phi &= \!R(\cW)(\@\Theta_\flt(\!X)\@\Phi) \\
    &= \!R(\cW) \left[\@\Theta_\flt(\!X) \star \@\varphi\right] \\
    &= \!R(\cW)\!L(\@\Theta(\!X,\@\varphi)) \\
    &= \cW \tcont \@\Theta(\!X,\@\varphi)
\end{align*}
Which gives \cref{eqn:TT-WSINDy=WSINDy}. The conclusion follows from \cref{rmk:MANDy=FlatStrongFormProblem}; the weak form of $\dot{\!X} = \cW \tcont \@\Theta(\!X)$ is obtained by multiplying $\@\Phi$ on both sides:
\begin{equation*}
    -\!X\dot{\@\Phi} = (\cW \tcont \@\Theta(\!X))\@\Phi = \cW_\flt \@\Theta_\flt(\!X) \@\Phi = \cW \tcont \@\Theta(\!X,\@\varphi)
\end{equation*}
\end{proof}
We have shown that TT-WSINDy solves the same weak regression problem as WSINDy. In the next section, we show that TT-WSINDy yields significant computation and memory savings when discovering high-dimensional ODEs.
 
\section{TT-WSINDy outperforms WSINDy in many dimensions\label{sec:TTwsindySpeed}}
In this section, we will demonstrate that TT-WSINDy can circumvent the curse of dimensionality. We will show that TT-STLS has both a time and memory complexity that is polynomial in $J,D$, and $M$, in contrast to the matrix ordinary least squares solve. We then show that under reasonable assumptions on sparsity of true support in the total search space, we may expect that the size of the problem is significantly reduced by TT-STLS. We also give an alternative algorithm for construction of the weak feature tensor when $M$ is large.

\subsection{TT-STLS time and memory complexity}
The efficiency of the TT-STLS algorithm relies on the ability to perform all requisite operations in the TT format, never forming the full, size $\cO(J^DM)$ tensor. This is reliant on the following fact.
\begin{lemma}
    Let $\cA, \cB \in \R^{I_1 \tct I_{N+1}}$ be rank $(A_1,\dots,A_N)$ and $(B_1,\dots,B_N)$ tensor trains, respectively. Then, the Hadamard product $\cC = \cA \odot \cB$ is a rank-$(A_1B_1,\dots,A_NB_N)$ tensor train. Namely, we define $\cC$ in terms of slices of the cores of $\cA,\cB$ as
    \[\cC^{(n)}[:,i_n,:] = \cA^{(n)}[:,i_n,:] \otimes \cB^{(n)}[:,i_n,:], \quad n \in [N+1], i_n \in [I_n]\]
    where $\otimes$ denotes the Kronecker product. \cite{Oseledets2011SIAMJSciComput} \label{lem:TTHadamard}
\end{lemma}
Using this, we have that TT structure is preserved under masking.
\begin{corollary}
    Every feature tensor iterate $\@\Psi_{\ell}$ is a rank-$(M,\dots,M)$ tensor train, with cores given by
    \[\@\Psi_\ell^{(d)} = \left\llbracket \begin{matrix}
        \widetilde{\@\theta}_\ell^{(d)}(t_1) && 0 \\
        &\ddots \\
        0&& \widetilde{\@\theta}_\ell^{(d)}(t_M)
    \end{matrix}\right\rrbracket \in \R^{M\times J \times M}\]
    for $d \in [D]$, where
    \begin{equation*}
    \begin{split}
        &\widetilde{\@\theta}^{(d)}_\ell(t_m) \in \R^J \\
        & \text{s.t. } \widetilde{\@\theta}^{(d)}_\ell(t_m)[j] = \begin{cases}
            f_j(x_{d,m}) & j \in S^{(d)}_{\ell} \\
            0 & \text{otherwise}
        \end{cases}
    \end{split}
    \end{equation*}
\end{corollary}
\begin{proof}
Let $\cM_\ell = \cM(S_\ell)$. Note that $\@\Psi_\ell := \@\Psi_{\ell-1} \odot \cM_{\ell} = \@\Theta(\!X,\@\varphi)\odot \cM_\ell$, following from the fact that the masks $\cM_\ell \subset \dots \subset \cM_2 \subset \cM_1$ are nested. And, the mask tensor can be written as a rank-$(1,\dots,1)$ tensor train:
\begin{equation*}
    \cM_{\ell} = \text{TT}\big\llbracket \!m^{(1)}_{\ell},\dots, \!m^{(D)}_{\ell}, \@1_{M'} \big\rrbracket
\end{equation*}
Then, the result follows immediately from \cref{lem:TTHadamard}. \end{proof}

Therefore each $(\@\Psi_\ell)^\dag$ may be computed via a TT pseudoinverse. \cref{lem:TTPI-runtime} then gives that each pseudoinverse in TT-STLS can be computed in $\cO(JDM^3)$ time. And the computational effort of obtaining the Hadamard product $\@\Theta(\!X,\@\varphi) \odot \cM_{\ell}$ amounts to selectively zeroing out elements of $\@\Theta(\!X,\@\varphi)$, which can be no worse than the $\cO(JDM^2)$ storage size of the weak feature tensor.

It remains to discuss the time complexity of computing the coarse support of a coefficient estimate $\cW_\ell$. This computation being polynomial in the dimension parameters of $\cW$ is nontrivial, as the coarse support is taken with respect to the fully-formed tensor representation of $\cW$. If the support were taken only from the entries of the tensor train, we would not observe the effects of regressing against the LHS vector $\!y$, as $\!y \tcont \big(\@\Psi_{\ell-1}\big)^\dag = \cW_\ell$ only contracts $\!y$ into one core of $\big(\@\Psi_\ell\big)^\dag$.

We will now give an algorithm that computes coarse-supp$(\cW_\ell, \lambda)$ with no worse time or memory complexity than the TT pseudoinversion algorithm. To describe this algorithm, we will describe a TT in terms of slices of its cores, rather than vector fibers of its cores, as we have been doing until now. Let
\begin{equation}
    \cW = \text{TT}\big\llbracket \cW^{(1)},\dots,\cW^{(D)}\big\rrbracket \in \R^{J_1\tct J_D}
\end{equation}
where $\cW^{(d)} \in\R^{M_{d-1} \times J_d \times M_d}$. Then, denote the matrix core slices
\begin{equation}
    \!W^{(d)}[j] = \cW^{(d)}[:,j,:], \quad d \in [D], j \in [J_d]
\end{equation}
Now, we may index $\cW$ by the matrix product. The matrix product below evaluates to a scalar, since $\!W^{(1)}[j_1] \in \R^{1\times M_1}$ and $\!W^{(D)}[j_D] \in \R^{M_{D-1}\times 1}$.
\begin{equation}
    \cW[j_1,j_2,\dots,j_D] = \!W^{(1)}[j_1]\!W^{(2)}[j_2]\cdots\!W^{(D)}[j_D]
\end{equation}
Using this, we will formulate an alternate characterization of the coarse support of $\cW$ that relies only on matrix products. First, define the \textit{left/right accumulation matrices} as
\begin{equation}
\begin{split}
    \@\Gamma_L^{(d)} \in \R^{\left[\prod_{i=1}^d J_i\right] \times M_d}, \quad &\big(\@\Gamma_L^{(d)}\big)\left[\overline{j_1,\dots,j_d},\alpha\right] = \bigg(\!W^{(1)}[j_1]\cdots\!W^{(d)}[j_d]\bigg)[\alpha] \\
    \@\Gamma_R^{(d)} \in \R^{M_{d-1} \times \left[ \prod_{i=d}^{D}J_i \right]}, \quad &\big(\@\Gamma_R^{(d)}\big)\left[\beta,\overline{j_d,\dots,j_D}\right] = \bigg( \!W^{(d)}[j_d]\cdots \!W^{(D)}[j_D] \bigg)[\beta]
\end{split}
\end{equation}
and the \textit{left/right density matrices} as
\begin{equation}
    \!D_L^{(d)} = (\@\Gamma_L^{(d)})^T\@\Gamma_L^{(d)}, \quad \!D_R^{(d)} = \@\Gamma_R^{(d)}(\@\Gamma_R^{(d)})^T 
\end{equation}
Then, our characterization of the coarse support is as follows.
\begin{lemma}
    Let $\cW \in \R^{J_1\tct J_D}$ be a rank-$(M_1,\dots,M_{D-1})$ tensor train, $\lambda \in \R$, and 
    $S = \{S^{(1)},\dots,S^{(D)}\} = \text{coarse-supp}(\cW, \lambda)$.
    Then, we have that, for $d \in [D], j \in [J_d]$,
    \begin{equation}
        \left( j \in S^{(d)} \right) \iff \left( \text{tr}\big[ (\!W^{(d)}[j])^T\!D_L^{(d-1)}\!W^{(d)}[j] \!D_R^{(d+1)} \big] \ge \lambda \right)
        \label{eqn:coarse-suppCondition}
    \end{equation}
\end{lemma}
\begin{proof}
The result follows from manipulating the trace using standard identities
\begin{align*}
    \big\|\cW[:\cdots:,j,:\cdots:]\big\|_F^2 &= \big\|\@\Gamma_L^{(d-1)}\!W^{(d)}[j]\@\Gamma_R^{(d+1)}\big\|_F^2 \\
    &= \text{tr}\big[ (\@\Gamma_R^{(d+1)})^T(\!W^{(d)}[j])^T(\@\Gamma_L^{(d-1)})^T\@\Gamma_L^{(d-1)}\!W^{(d)}[j] \@\Gamma_R^{(d+1)} \big] \\
    &= \text{tr}\big[(\!W^{(d)}[j])^T\!D_L^{(d-1)}\!W^{(d)}[j] \@\Gamma_R^{(d+1)}(\@\Gamma_R^{(d+1)})^T \big] \\
    &= \text{tr}\big[ (\!W^{(d)}[j])^T\!D_L^{(d-1)}\!W^{(d)}[j] \!D_R^{(d+1)} \big]
\end{align*}
\end{proof}

So computing the coarse support of $\cW$ amounts to checking the condition \cref{eqn:coarse-suppCondition} for each $d \in [D], j \in [J_d]$. And the left/right density matrices can be precomputed using the following recursions.
\begin{equation}
\begin{split}
    \!D^{(d)}_L &= \sum_{j_d=1}^{J_d} \big(\!W^{(d)}[j_d]\big)^T\!D^{(d-1)}_L\!W^{(d)}[j_d], \quad \!D^{(0)}_L = 1 \\
    \!D^{(d)}_R &= \sum_{j_d=1}^{J_d} \!W^{(d)}[j_d]\!D^{(d+1)}_R\big(\!W^{(d)}[j_d]\big)^T, \quad \!D^{(D+1)}_R = 1
\end{split}
\end{equation}
\Cref{alg:coarse-supp} details the full routine, and we give the full time/memory complexities of TT-STLS in \cref{prop:TT-STLSCost}. Crucially, both are polynomial in the shape parameters of the weak feature tensor, thus subverting the curse of dimensionality.
\begin{algorithm}[htb]
\caption{\label{alg:coarse-supp} coarse-supp}
\begin{algorithmic}[1]
    \REQUIRE \quad Tensor train $\cW \in \R^{J_1\tct J_D}$ \\
        \quad\quad\quad Thresholding parameter $\lambda \in \R$
    \ENSURE Coarse support $S \subset [J]^D$
    \STATE Set initial values $\!D^{(0)}_L, \!D_R^{(D+1)} \leftarrow 1$
    \STATE Accumulate $\!D_R^{(D)},\!D_R^{(D-1)},\dots,\!D_R^{(2)}$ 
    \FOR {$d = 1,\dots, D$}
        \FOR{$j_d = 1,\dots,J_d$}
            \STATE Compute $\!A_{j_d}^{(d)} \leftarrow (\!W^{(d)}[j_d])^T\!D_L^{(d-1)}\!W^{(d)}[j_d]$
            \IF{$\text{tr}\big[ \!A_{j_d}^{(d)} \!D_R^{(d+1)} \big] \ge \lambda$}
                \STATE Append $j_d$ to $S^{(d)}$
            \ENDIF
        \ENDFOR
        \STATE Compute $\!D_L^{(d)} \leftarrow \sum_{j_d=1}^{J_d} \!A^{(d)}_{j_d}$
    \ENDFOR
    \RETURN $S := \{S^{(1)},\dots,S^{(D)}\}$
\end{algorithmic}
\end{algorithm}

\begin{proposition}
    The computational complexity of \cref{alg:TT-STLS} can be estimated as \\ $\cO((JD)^2M^3)$, and the memory complexity can be estimated as $\cO(JDM^2)$. \label{prop:TT-STLSCost}
\end{proposition}
\begin{proof}
The cost of computing all necessary density matrices for the coarse support is $\cO(JDM^3)$. Given these, we compute $\cO(JD)$ traces that each cost $\cO(M^3)$, for a total of $\cO(JDM^3)$. Therefore the $\cO(JDM^3)$ time complexity of computing the coarse support is no worse than the time complexity of computing the TT pseudoinverse, and the full cost of one TT-STLS loop is $\cO(JDM^3)$. The worst case number of iterations is $JD$, for a total time complexity of $\cO((JD)^2M^3)$.

The memory complexity of coarse-supp is dominated by the cost of storing $\!D_R^{(D)},\dots,\!D_R^{(2)}$. To accumulate these, we need never store more than a single $\!W^{(d)}[j_d]$ and intermediate $\cO(M^2)$ matrix. And each $D_{L/R}^{(d)}$ consumes $\cO(M^2)$ storage, so the memory cost of all right density matrices is $\cO(DM^2)$. In computing the traces, we need never store more than a single $\!A_{j_d}^{(d)}, \!W^{(d)}[j_d], \!D_L^{(d-1)}$, and an intermediate $\cO(M^2)$ matrix at a time. Therefore, coarse-supp has a memory complexity of $\cO(DM^2)$. In the TT-STLS loop, we only need store a single $\@\Psi_{\ell}, \big(\@\Psi_\ell\big)^\dag, \cW_\ell, S_\ell$ in addition to the cost of coarse-supp, which all have cost not greater than $\cO(DJM^2)$.
\end{proof}

\begin{remark}
    The $\cO(JD)$ bound on iterations is loose. In practice, the time complexity of TT-STLS is $\cO(JDM^3)$. \label{rmk:iterationsinpractice}
\end{remark}

\begin{remark}
    Assuming a constant number of iterations, the computational complexity of matrix STLS is $\cO(J^DM^2)$. This implies that for  large $M$, say $M \gtrsim J^{D}$, matrix STLS will be preferable in runtime to TT-STLS.
\end{remark}

\subsection{Low-rank feature tensor}
We can often do much better than the complexities given by \cref{prop:TT-STLSCost}, which scale well with respect to the search space shape parameters $J,D$, but poorly with respect to the number of time snapshots $M$. Recall that the weak feature tensor is constructed as a rank-$(M,\dots,M)$ tensor train. However, in certain cases, particularly for large $M$, the true ranks of the feature tensor are much lower. Consider the following lemma.
\begin{lemma}\label{lem:TTranks}
    Let $\cX \in \R^{I_1\tct I_N}$. If, for each mode-$n$ unfolding $\cX_{(n)}$ of $\cX$, we have
    \begin{equation*}
        \text{rank}\left(\cX_{(n)}\right) \leq R_n
    \end{equation*}
    Then there exists a TT decomposition $\cT$ of $\cX$ with ranks not higher than $(R_1,\dots,R_{N-1})$. \cite{Oseledets2011SIAMJSciComput}
\end{lemma}
Now, let $\@\Theta_{(d)}$ denote the mode-$d$ unfolding of the weak feature tensor. We have
\begin{equation*}
    \@\Theta_{(d)} \in \R^{J^d \times J^{D-d}M'}, \quad d \in [D]
\end{equation*}
Based on the shape of the unfolding, we immediately have $\text{rank}\left(\@\Theta_{(d)}\right) \leq \min\{J^d, M'J^{D-d}\}$. Depending on the size of the search space and number of snapshots, this can show that the true ranks of $\@\Theta(\!X, \@\varphi)$ are significantly smaller than the original construction suggests. For example, when $J = 3, D = 10$, and $M' \approx M = 6000$, we see in \cref{tab:lowrankbound} that the first $7$ ranks are massively reduced, cutting the total storage from $\approx 1.0\cdot10^9$ to $\approx 7.6\cdot 10^8$.

We now give a memory-efficient algorithm that constructs the low-rank tensor train without forming the original weak feature tensor in memory.

\begin{algorithm}[htb]
\caption{\label{alg:lowrank_featuretensor} Rank-reduced feature tensor construction}
\begin{algorithmic}[1]
    \REQUIRE \quad Data matrix $\!X \in \R^{D\times M}$ \\
        \quad\quad\quad Family of basis functions $\@f = \{f_j\}_{j\in[J]}$ \\
        \quad\quad\quad Test function discretization $\@\varphi$ \\
        \quad\quad\quad Truncation parameter $\epsilon > 0$
    \ENSURE Rank-reduced feature tensor $\@\Theta(\!X,\@\varphi) \in \R^{J\tct J\times M'}$, with ranks $(R_1,\dots,R_D)$
    \STATE \COMMENT{Store basis feature data}
    \FOR {$j \in [J], d \in [D], m \in [M]$}
        \STATE Evaluate $\!B_d[j,m] \leftarrow f_j(x_{d,m})$
    \ENDFOR
    \STATE \COMMENT{Compute first D-1 cores}
    \STATE Set $\!C^{(1)} \leftarrow \@1^{1\times M}$
    \FOR{$d = 1,\dots, D-1$}
        \STATE Compute $\!E^{(d)} \leftarrow \!C^{(d)} \circledcirc \!B_d$ 
        \STATE Compute $\!U^{(d)}\@\Sigma^{(d)}(\!V^{(d)})^T \leftarrow \text{SVD}(\!E^{(d)})$
        \STATE Truncate $\!U^{(d)},\@\Sigma^{(d)},\@V^{(d)}$, using $\epsilon$
        \STATE Set $\widetilde{\@\Theta}^{(d)} := \!L^{-1}(\!U^{(d)}) \in \R^{R_{d-1}\times J \times R_d}$
        \STATE Iterate $\!C^{(d+1)} \leftarrow \@\Sigma^{(d)}(\!V^{(d)})^T \in \R^{R_d\times M}$
    \ENDFOR
    \STATE \COMMENT{Compute final two cores}
    \STATE Compute $\!E^{(D)} \leftarrow \!C^{(D)}\circledcirc \!B_D$
    \STATE Compute $\!U^{(D)}\@\Sigma^{(D)}(\!V^{(D)})^T \leftarrow \text{truncated-SVD}(\!E^{(D)}\@\Phi, \epsilon)$
    \STATE Set $\widetilde{\@\Theta}^{(D)} := \!L^{-1}(\!U^{(D)}) \in \R^{R_{D-1} \times J \times R_D}$
    \STATE Set $\widetilde{\@\Phi} := \!R^{-1}(\@\Sigma^{(D)}(\!V^{(D)})^T) \in \R^{R_D \times M' \times 1}$
    \RETURN $\text{TT}\big\llbracket \widetilde{\@\Theta}^{(1)},\dots,\widetilde{\@\Theta}^{(D)},\widetilde{\@\Phi}\big\rrbracket$
\end{algorithmic}
\end{algorithm}

\begin{lemma}
    \Cref{alg:lowrank_featuretensor} computes a representation of $\@\Theta(\!X,\@\varphi)$, with ranks
    \begin{equation*}
        R_d \leq \min\{J^d,M'J^{D-d}\}, \quad d\in [D]
    \end{equation*}
\end{lemma}

\begin{proof}
We will prove by comparing \cref{alg:lowrank_featuretensor} with the TT pseudoinverse algorithm, given in full detail in \cref{sec:TT-PI}. We first show that the first $D-1$ recomputed cores $\widetilde{\@\Theta}^{(1)},\dots,\widetilde{\@\Theta}^{(D-1)}$ are identical in both algorithms. This amounts to showing that at each of the first $D-1$ iterations, both algorithms compute an SVD of the same matrix. The first iteration computes an SVD of the unmodified first core $\@\Theta^{(1)}$. And since $\!C^{(1)} = \!1^{1\times M}$, we have
\begin{equation*}
    \!E^{(1)} = \!C^{(1)} \circledcirc \!B_1 = \!B_1 \in \R^{J \times M}
\end{equation*}
where $\circledcirc$ is the Khatri-Rao\footnote{Note that the use of the Khatri-Rao product is dependent on $b_2$ being row-major indexing, used by default in NumPy.} product \cite{KoldaBader2009SIAMRevb}. By construction, $\!B_1 = \@\Theta^{(1)}$. Now, note that at the $d^\text{th}$ iteration, the TT pseudoinverse algorithm computes an SVD of the matrix
\begin{equation*}
    \!L\big(\@\Sigma^{(d-1)}(\!V^{(d-1)})^T \tcont \@\Theta^{(d)}\big)
\end{equation*}
Suppose for induction that the matrices $\widetilde{\@\Theta}^{(1)},\dots,\widetilde{\@\Theta}^{(d-1)}$ agree, which gives that $\!C^{(d)} = \@\Sigma^{(d-1)}(\!V^{(d-1)})^T$. Then, we have 
\begin{align*}
    \!L\big(\@\Sigma^{(d-1)}(\!V^{(d-1)})^T \tcont \@\Theta^{(d)}\big)[\overline{r_{d-1},j},m] &= \sum_{\widetilde{m}=1}^M \big(\@\Sigma^{(d-1)}(\!V^{(d-1)})^T\big)[r_{d-1},\widetilde{m}]\@\Theta^{(d)}[\widetilde{m},j,m] \\
    &= \big(\@\Sigma^{(d-1)}(\!V^{(d-1)})^T\big)[r_{d-1},m]f_j(x_{d,m}) \\
    &= \!C^{(d)}[r_{d-1},m]\!B_d[j,m] \\
    &= (\!C^{(d)} \circledcirc \!B_d)[\overline{r_{d-1},j},m]
\end{align*}
Therefore the first $D-1$ cores computed by both algorithms are the same. For the final two cores $\widetilde{\@\Theta}, \widetilde{\@\Phi}$, we deviate from the original pseudoinverse procedure by applying the weak-form matrix $\@\Phi$ to $\!E^{(D)}$, and forming both orthonormalized cores from a single SVD. Using  \cref{lem:unfoldingproductequalscontraction} and our inductively shown formula for $\!E^{(D)}\@\Phi$, we have that this is equivalent to contracting the final two cores, and taking an SVD.
\begin{align*}
    \!L^{-1}(\!E^{(D)}\@\Phi) &= \!L^{-1}\big(\!L(\@\Sigma^{(D-1)}(\!V^{(D-1)})^T\tcont \@\Theta^{(D)})\@\Phi\big) \\
    &= \big(\@\Sigma^{(D-1)}(\!V^{(D-1)})^T\tcont \@\Theta^{(D)}\big)\tcont \@\Phi
\end{align*}
So by computing an SVD of $\!E^{(D)}\@\Phi$ and separating its left- and right-orthonormal components, we do indeed compute an orthonormal representation of the weak feature tensor with total storage cost
\begin{equation}
    \cO\left(JR_1 + J\sum_{d=1}^{D-1}R_dR_{d+1} + R_DM'\right)
\end{equation}
over the original $\cO(JDM^2)$ cost.
\end{proof}

The final step above can save significant time and memory: since $R_D = \cO(\min\{J^D,M'\})$, the naive approach risks taking an SVD of two size $\cO(J^DM)$ matrices, which automatically puts the time complexity of the pseudoinverse computation above the cost of a matrix pseudoinverse (which requires only one $\cO(J^DM)$ SVD). Additionally, the cross-correlation
\begin{equation*}
    \!E^{(D)} \star \@\varphi = \!E^{(D)}\@\Phi \in \R^{R_{D-1}J\times M'}
\end{equation*}
acts as a low-pass filter and can significantly reduce the true rank of $\!E^{(D)}$ prior to computing the SVD. To exploit this, we compute $\text{truncated-SVD}(\!E^{(D)}\@\Phi)$ using the adaptive-rank randomized range finder proposed by Halko, Martinsson, and Tropp \cite{halkoFindingStructureRandomness2011}. In the example in \cref{tab:lowrankbound}, this further reduces the total storage to $\approx 1.1\cdot 10^8$, a full order of magnitude from the original cost.

\begin{table}[h!]
\footnotesize
\caption{The ranks of the weak feature tensor with shape parameters $J = 3, D = 10, M' \approx M = 6000$ and data from the Lorenz 96 model, via both the original construction and the upper bound obtained using \cref{lem:TTranks}.}
\centering
\begin{tabular}{|c | c c c c c c c c c c |} 
     \hline
     $d$ & 1 & 2 & 3 & 4 & 5 & 6 & 7 & 8 & 9 & 10 \\ [0.5ex] 
     \hline
     original ranks & 6000 & 6000 & 6000 & 6000 & 6000 & 6000 & 6000 & 6000 & 6000 & 6000  \\\hline
     $\min\{J^d,M'J^{D-d}\}$ & 3 & 9 & 27 & 81 & 243 & 729 & 2187 & 6561 & 18000 & 6000 \\\hline 
     \cref{alg:lowrank_featuretensor} & 3 & 9 & 27 & 81 & 243 & 729 & 2187 & 3276 & 3544 & 3068 \\\hline 
\end{tabular}
\label{tab:lowrankbound}
\end{table}

\Cref{alg:lowrank_featuretensor} also can yield substantial walltime improvements. \Cref{fig:lowrank_walltime_wrt_M} shows the scaling of the original and low-rank constructions in $M$ and $D$. We observe the tradeoff in opting for the low-rank construction: we exchange some time scaling in $D$ for improved time scaling in $M$ and smaller storage cost. We emphasize that the choice of method is dependent on the shape parameters $J,D,M$, and that neither is optimal in all settings.

\begin{figure}[htb]
    \centering
    \includegraphics[width=\linewidth]{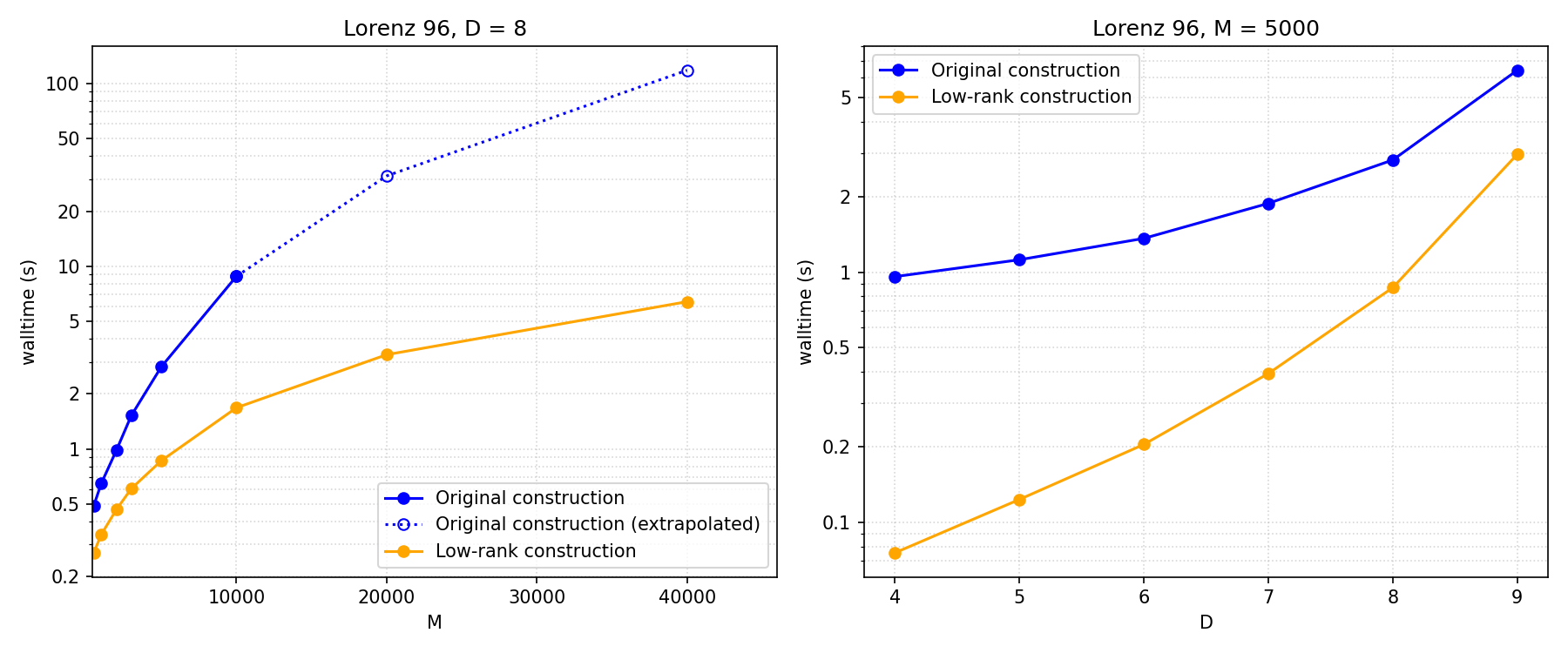}
    \caption{Walltimes of TT-WSINDy using the original construction and the low-rank construction, applied to the Lorenz 96 model over basis features $\@f = \{1,x\}$. The left plot fixes $D = 8$ and times over varying $M$, while the right plot fixes $M = 5 000$ and plots over varying $D$. At $D=8, M = 20 000, 40 000$, the original weak feature tensor was too large to be stored in memory, and the walltime is extrapolated from previous data points.}
    \label{fig:lowrank_walltime_wrt_M}
\end{figure}

\subsection{TT-STLS sparsification}
After execution of TT-MSTLS, TT-WSINDy calls the matrix MSTLS algorithm. Therefore it is crucial to the efficiency of TT-WSINDy that the size of the problem has been substantially reduced by the time it is handed off to the flat solver. There exist pathological cases for which this is not true. For example, if every atomic term $f_j(x_d)$ appears in the true ODE, the MSTLS time complexity reverts to $\cO(J^DM)$ and we lose the benefits of TT-WSINDy. To guard against this, we assume that the true terms of \cref{eqn:linearInParameters} are sparse in the space of basis functions.

\begin{remark}
    \label{rmk:sparsityofbasisfuctions} Say the $d^{\text{th}}$ equation of a $D$-dimensional system of ODEs has true basis terms $\@f^\star_d \subset \@f$, which includes $p_d \leq J$ basis functions applied to $q_d \leq D$ dimensions. We assume that, for all $d \in [D]$
    \begin{equation}
        |\@g_d^\star| = (p_d)^{q_d} \ll J^D = |\@g|
    \end{equation} 
    where $\@g^\star_d \subset \@g$ is the library induced by the basis terms $\@f_d^\star$.
\end{remark}

In other words, the minimal coarse support is substantially smaller than the initial, full support. This assumption is mild due to the exponential size of the total search space of functions. However, if the above remark does not hold, we do not expect TT-WSINDy to be preferred to matrix WSINDy, as the TT-MSTLS stage becomes superfluous.

We have not yet demonstrated explicitly that the masking procedure in TT-STLS reduces the size of the linear problem. We will do so now. First, we show the following lemma.
\begin{lemma}
    Let $\cT \in \R^{I_1\tct I_D \times K}$ be a tensor train with ranks $(R_1,\dots,R_D)$. If the slice \text{ } $\cT^{(d)}[:,i^*,:] = \@0$, then the same slice of the corresponding core of the pseudoinverse $\cT^\dag \in \R^{K \times I_1 \tct I_D}$ is zero.\label{lem:TTPIzeroing}
\end{lemma}
\begin{proof}
Let $\widetilde{\cT}^{(D)}, \widetilde{\cT}^{(1)}, \dots, \widetilde{\cT}^{(D-1)}$ denote the cores of the pseudoinverse $\cT^\dag$, as in \cref{eqn:TTPseudoinverse}. We say that each $\cT^{(d)}, \widetilde{\cT}^{(d)}$ correspond. We may derive an explicit formula for $\widetilde{\cT}^{(d)}$:
\begin{equation*}
    \widetilde{\cT}^{(d)} = (\!R^{(d-1)} \tcont \cT^{(d)}) \tcont (\!R^{(d)})^\dag \in \R^{S_{d-1} \times I_d \times S_d}
\end{equation*}
where $\!R^{(d-1)} := \@\Sigma^{(d-1)}(\!V^{(d-1)})^T \in \R^{S_{d-1}\times R_{d-1}}$ is the upper-triangular component of the QR decomposition of the (left-unfolded) previous core, and $\!R^{(d)} \in \R^{S_d \times R_d}$ is the same, but of $\!R^{(d-1)} \tcont \cT^{(d)}$. Suppose that $\cT^{(d)}[:,i^*,:] = \@0$. We may write the slice $\cT^{(d)}[:,i^*,:]$ as tensor contraction with a Euclidean basis vector $\@e_{i^*}$ in the mode dimension. Letting $\times_2$ denote the 2-mode tensor-vector product \cite{KoldaBader2009SIAMRevb},
\begin{equation*}
    \cT^{(d)}[:,i^*,:] = \cT^{(d)} \times_2 \@e_{i^*} = \@0 \in \R^{R_{d-1} \times R_d}
\end{equation*}
Now we check whether this holds for $\widetilde{\cT}^{(d)}$. Expanding all contractions, we have
\begin{align*}
    (\widetilde{\cT}^{(d)} \times_2 e_{i^*})[s_{d-1},s_d] &= \sum_{i=1}^{I_d} \left[ \sum_{r_{d-1}=1}^{R_{d-1}}\sum_{r_d=1}^{R_d} \!R^{(d-1)}[s_{d-1},r_{d-1}]\cT^{(d)}[r_{d-1},i,r_d](\!R^{(d)})^\dag[r_d,s_d] \right]\@e_{i^*}[i] \\
    &= \sum_{r_{d-1}=1}^{R_{d-1}}\sum_{r_d=1}^{R_d} \!R^{(d-1)}[s_{d-1},r_{d-1}]\cT^{(d)}[r_{d-1},i^*,r_d](\!R^{(d)})^\dag[r_d,s_d]
\end{align*}
Which equals zero for all $s_{d-1},s_d$, since $\cT^{(d)}[r_{d-1},i^*,r_d] = 0$ for all $r_{d-1},r_d$.
\end{proof}

\begin{corollary}\label{cor:maskapplicationlimitssupport}
    Let $\@\Theta(\!X,\@\varphi)$ be the weak feature tensor, and let $\cM = \cM(S) \in \R^{J\tct J\times M'}$ be a mask obtained from some coarse support $S \subset [J]^D$. Let $\@\Theta|_{S} := \@\Theta(\!X,\@\varphi) \odot \cM$, and $\cW := \!y \tcont (\@\Theta|_S)^\dag$. It then follows that
    \begin{equation*}
        \big(j \notin S^{(d)}\big) \Rightarrow \left( \cW\bigg[\overbrace{:\cdots :}^{d-1},j,\overbrace{:\cdots:}^{D-d}\bigg] = \@0 \right)
    \end{equation*}
\end{corollary}
\begin{proof}
By construction of the mask, $j \notin S^{(d)}$ implies $(\@\Theta|_S)^{(d)}[:,j,:] = \@0$. Then, observing that
\begin{equation*}
    \cW = \!y \tcont (\@\Theta|_S)^\dag = \sum_{m'=1}^{M'} \!y_{m'}(\@\Theta|_S)^\dag[m',:\cdots:]
\end{equation*}
we have that the result follows immediately from \cref{lem:TTPIzeroing}, as $(\@\Theta|_S)^\dag[m',:\cdots:,j,:\cdots:] = \@0$ for all $m' \in [M']$.
\end{proof}

It follows that, given support $S \subset [J]^D$ with $J_d := |S^{(d)}|$ basis features in each dimension, we may reduce the feature tensor to some $\widetilde{\@\Theta} \in \R^{J_1 \tct J_D \times M'}$ that implicitly enforces $\cW$ has at most the same size support as the previous iteration. We conclude with the following corollary, which follows immediately from \cref{cor:maskapplicationlimitssupport}.

\begin{corollary}
    The modes of $\@\Psi_\ell$ are monotonically decreasing in $\ell$.
\end{corollary}

In the next section, we will observe empirically that TT-MSTLS outputs a much smaller linear problem than it is given. We reserve a more detailed investigation of the reduction in size of the problem for future work.

\section{Examples and results\label{sec:results}}
\subsection{Weak vs. Strong Form TT Regression}
Despite being a lightweight modification to the strong feature tensor,  the weak core grants significant robustness to noise in the ensuing regression computation. In this subsection, we compare weak- and strong-form pseudoinverse computations. We omit sparsification to isolate the effect of the weak form.

For four systems of ODEs, we are given data $\!X \in \R^{D\times M}$, where $\!X[d,m] = \!X^\star[d,m] + \epsilon_{d,m}$, $\!X^\star[d,m] = \!x_d(t_m)$ is the true data and the noise $\{\epsilon_{d,m}\}_{d\in[D],m\in[M]}$ is drawn i.i.d. from $\mathcal{N}(0,\sigma^2)$. The variance $\sigma^2$ is induced by the \textit{noise ratio} $\sigma_{NR}$ as
\begin{equation}
    \sigma^2 = \frac{\sigma_{NR}^2\|\!X^\star\|_F^2}{MD}
\end{equation}
We then perform TT regression to obtain a coefficient estimate $\widehat{\cW}$, and plot the relative coefficient error with respect to the true coefficient tensor $\cW^\star$
\begin{equation}
    err(\widehat{\cW}) = \frac{\|\cW^\star - \widehat{\cW}\|_F}{\|\cW^\star\|_F}
\end{equation}
against $\sigma_{NR}$.

A key difference from \cite{GelssKlusEisertEtAl2019JComputNonlinearDyn} is that we do not assume the ability to directly sample derivative data -- to compute the strong form coefficient estimate, we calculate $\dot{\!X}$ using finite differences. In the weak form, we take $\@\varphi, \dot{\@\varphi}$ to be obtained from the piecewise polynomial test function $\varphi$, given by
\begin{equation}
    \varphi(t) = \begin{cases}
        C(r - t)^p(r + t)^p & t \in [-r,r] \\
        0 & \text{otherwise}
    \end{cases}\label{eqn:polytestfn}
\end{equation}
where $p$ is the degree of the polynomial, $r$ is the radius of support, and $C$ is a normalization constant enforcing $\|\varphi\|_{L_2} = 1$. For details on choosing $p$ and $r$, see \cite{Tran.Bortz2026SIAMJSciComput}. For the second-order FPUT problem, we compute $\ddot{\!X}$ using second-order finite differences to obtain the strong form, and convolve $\!X$ against the second derivative $\ddot{\@\varphi}$ to obtain the weak form.

For the two chaotic systems, we sampled one trajectory, with $10^4$ snapshots for Lorenz 96 and $2\cdot10^5$ for Chua's circuit. For the non-ergodic systems, we sampled multiple trajectories: $6$ with $10^4$ snapshots each for FPUT, and $5$ with $10^4$ snapshots each for the Kuramoto model. For all experiments, we fixed $\Delta t = 0.1$. For robustness, we regressed against 40 different realizations at each noise level.

The ODEs used are the \textit{Fermi-Pasta-Ulam-Tsingou (FPUT) problem}, a quasi-periodic system of oscillators; the \textit{Lorenz 96 model}, a high-dimensional chaotic system; the \textit{Kuramoto model}, a high-dimensional system of coupled oscillators; and \textit{Chua's circuit}, a chaotic system with a nonlinear cross term. The exact dynamics and parameters used are given in \cref{tab:exampleODEs}, and the choice of basis functions and induced library size are given in \cref{tab:examplebasisFns}. Note that the identity $\sin(x\pm y) = \sin(x)\cos(y) \pm \cos(x)\sin(y)$ allows the Kuramoto model to be discoverable via the given three-term basis. Also note that for Chua's circuit to be discoverable without explicitly adding $f(x_d) = x_d|x_d|$ to the basis function library $\@f$, it is necessary to use a \textit{function-major} feature tensor, which we discuss in \cref{sec:functionMajor}.

\begin{table}[h]
\caption{The dynamics, parameters, and constraints for each of the four ODEs used.}
\footnotesize
  \centering
  \begin{tabular}{c c c}
    \hline
    Name & Dynamics & Parameters/constraints  \\\hline 
    && $D = 4$ \\
    FPUT & $\ddot{\!x}_d = (\!x_{d+1} -2\!x_d + \!x_{d-1}) + \beta\big[ (\!x_{d+1} - \!x_d)^3 - (\!x_d - \!x_{d-1})^3 \big]$ & $\beta = 0.7$ \\
    & & $\!x_0 = \!x_{D+1} = 0$ \\\hline
    && $D=5, F=8$ \\
    Lorenz 96 & $\dot{\!x}_d = (\!x_{d+1} - \!x_{d-2})\!x_{d-1} - \!x_d + F$ & $\!x_{-1} = \!x_{D-1}$ \\
    && $\!x_0=\!x_D, \!x_{D+1}=\!x_1$ \\\hline
    && $D=4$ \\
    Kuramoto model & $\dot{\!x}_d = \omega_d + \frac{K}{D}\sum_{d'=1}^D \sin(\!x_{d'} - \!x_d) + h\sin(\!x_d)$ & $K=2, h =\frac12,$ \\
    && $\omega_d = -5 + d\frac{10}D$ \\\hline
    & $\dot{\!x}_1 = \alpha(\!x_2 - (1+\delta_1)\!x_1 - \delta_2x_1|x_1|)$ & $\alpha=10,$ \\
    Chua's circuit & $\dot{\!x}_2 = \!x_1 - \!x_2 + \!x_3$ & $\beta=14.87,$ \\
    & $\dot{\!x}_3 = -\beta\!x_2$ & $\delta=(-\frac87, \frac{4}{63})$ \\\hline
  \end{tabular}
  \label{tab:exampleODEs}
\end{table}

\begin{table}[h]
\caption{The set of basis functions used for each ODE, the (per-equation) search space size induced by that set, and average number of true terms in each dimension.}
\footnotesize
  \centering
  \begin{tabular}{c c c c}
    \hline
    Name & Basis functions & Induced search space size & Mean library size  \\\hline 
    FPUT & $\{1, x, x^2, x^3\}$ & $4^D$ & $10 - 8/D$ \\\hline
    Lorenz 96 & $\{1, x\}$ & $2^D$ & $4$ \\\hline
    Kuramoto model & $\{1, \sin(x), \cos(x)\}$ & $3^D$ & $2D + 1$ \\\hline
    Chua's circuit & $\{1, x, |x|\}$ & $27$ & $2.\overline{3}$ \\\hline
  \end{tabular}
  \label{tab:examplebasisFns}
\end{table}

\Cref{fig:weakvstrongform_all} demonstrates that the weak form provides increased accuracy at both low and high noise levels. In the Lorenz 96, Kuramoto model, and Chua's circuit examples, there is a noise level for which the strong form estimates the dynamics comparably to the weak form. However, we see that the weak form provides robustness over the $[0,1]$ continuum of noise levels.

Due to the large datasets, we used the low-rank feature tensor construction given in \cref{alg:lowrank_featuretensor}. Even for the data-heavy Chua's circuit computation, regression terminated in an average of 1.33 seconds.

\begin{figure}
    \centering
    \includegraphics[width=\linewidth]{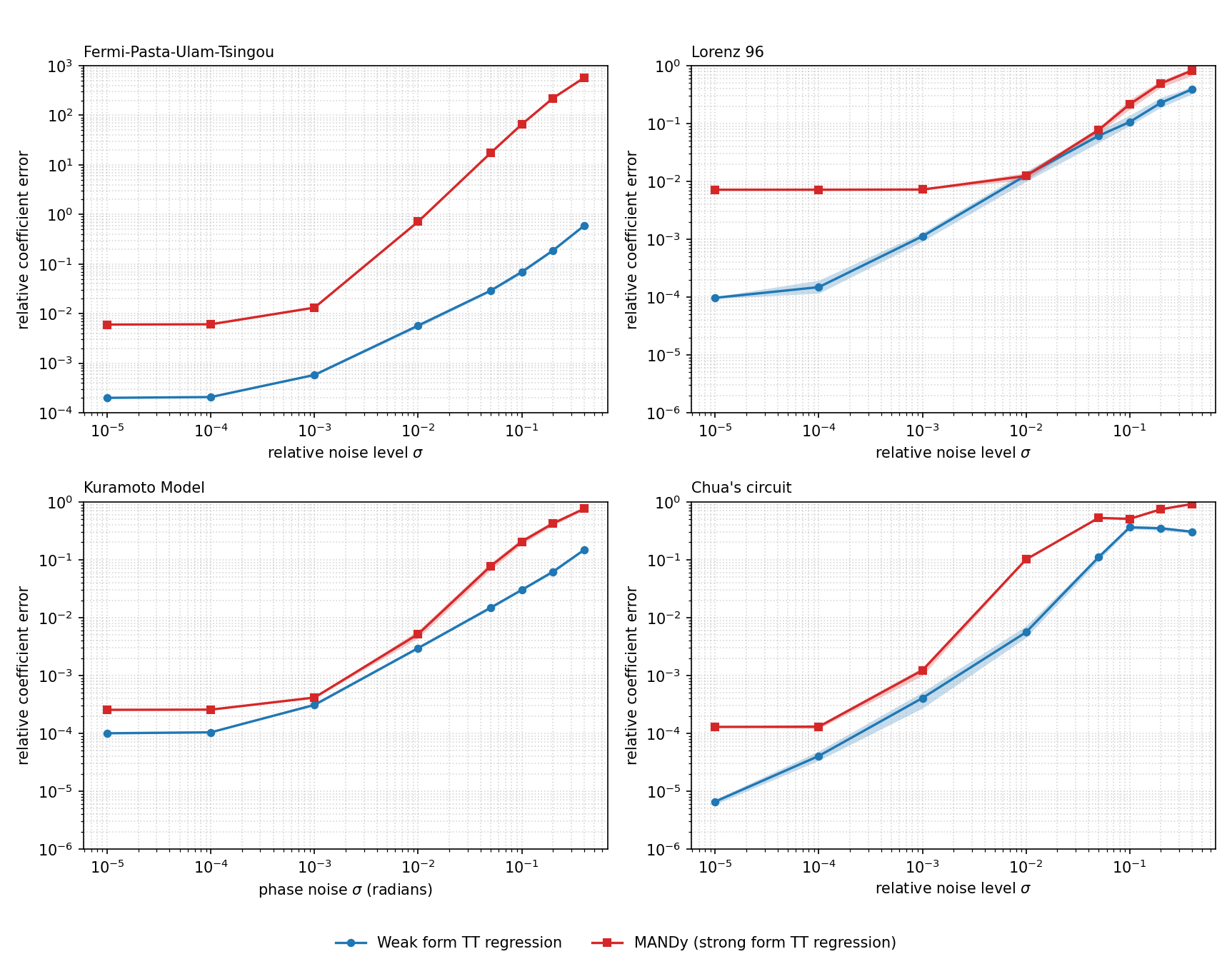}
    \caption{Weak form vs. strong form regression accuracy. Relative coefficient error is plotted against noise ratio for each of the four ODEs given in \cref{tab:exampleODEs}. The translucent bands give the upper and lower quartiles of error over 40 realizations. For a more detailed description of the variance across trials, see \cref{fig:weakvstrongform_boxplot}.}
    \label{fig:weakvstrongform_all}
\end{figure}

\subsection{TT-WSINDy Walltime}\label{sec:walltime}
To illustrate how TT-WSINDy scales in increasing $D$, we examine the Lorenz 96 model. To correctly identify the dynamics given in \cref{tab:exampleODEs}, we must include cross terms $x_ix_j$ between dimensions. If we want to check only the true basis terms $\@f = \{1,x\}$, then the induced search space is $|\@g| = 2^D$ per equation. In contrast, finding the true coarse support yields the following feature vectors for the $d^\text{th}$ dimension
\begin{equation}
    \@\theta^{(1)},\dots,\@\theta^{(D)} = \overbrace{[1],\dots,[1]}^{d-3}, \begin{bmatrix}
        1 \\ \!x_{d-2}
    \end{bmatrix}, \begin{bmatrix}
        1 \\ \!x_{d-1}
    \end{bmatrix}, \begin{bmatrix}
        1 \\ \!x_{d}
    \end{bmatrix}, \begin{bmatrix}
        1 \\ \!x_{d+1}
    \end{bmatrix},
    \overbrace{[1],\dots,[1]}^{D-d-1}
    \label{eqn:L96truecoarsesupp}
\end{equation}
which induces a size $2^4 = 16$ search space per equation. Given this support, the computational cost of matrix WSINDy is trivial. 

For $D = 5,\dots,12$, we fix $M = 20 000$, sufficiently many snapshots to identify the true support for all values of $D$. We corrupt with a modest noise ratio of $\sigma_{NR} = 10^{-3}$, and use the same test function as given in \cref{eqn:polytestfn}. Given that $M \gg J^D$, we again use the low-rank feature tensor construction.

In \cref{fig:lorenz96support}, we observe that TT-WSINDy does indeed limit the burden to the flat solver by handing off a significantly reduced support: for $D \leq 11$, the coarse pass identified the true coarse support described by equation \cref{eqn:L96truecoarsesupp}, after which the flat solve identified the true dynamics. For $D = 12$, a support containing around $2000$ spurious terms was identified, but the ensuing flat solve recovered the true dynamics. In all experiments, TT-WSINDy and flat WSINDy identified the true support. Both methods discovered identical coefficients to one another, with relative coefficient errors on the order of $10^{-4}$ at all values of $D$. 

\begin{figure}[htb]
    \centering
    \includegraphics[width=0.6\linewidth]{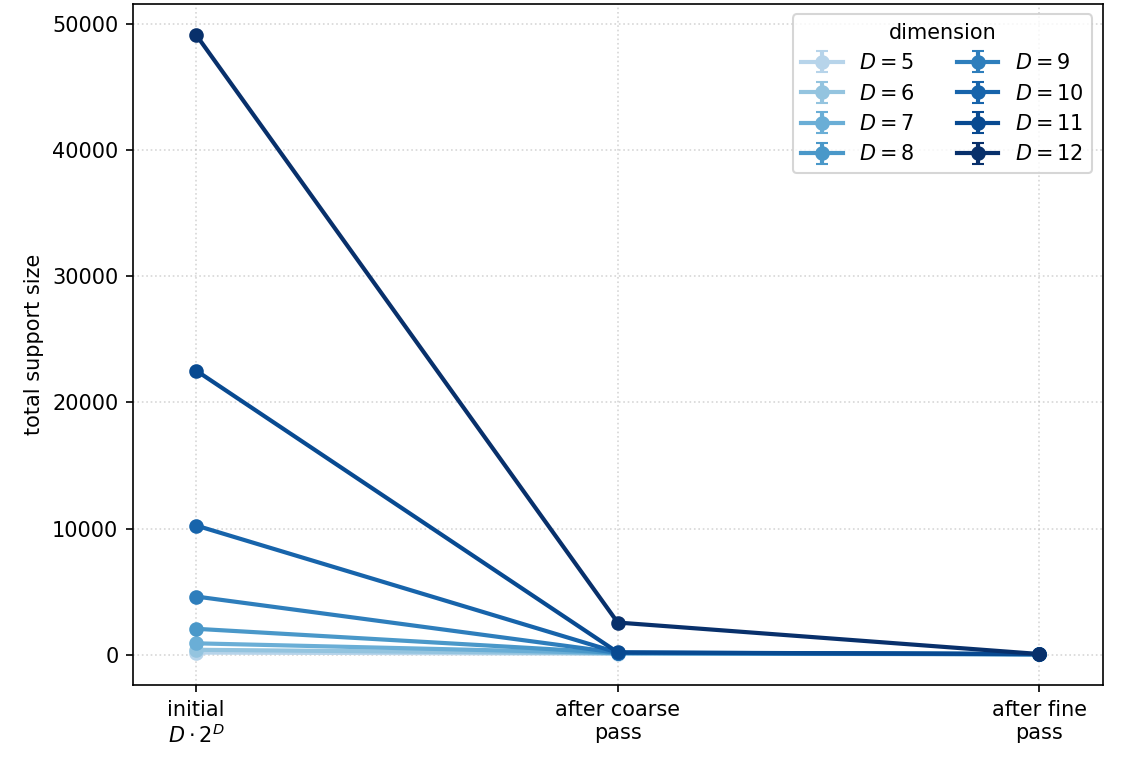}
    \caption{Visualization of the support reduction from the coarse to fine pass on Lorenz 96 for $D = 5,\dots,12$. The total support is the sum of the support of each equation in the system.}
    \label{fig:lorenz96support}
\end{figure}

\begin{figure}[htb]
    \centering
    \includegraphics[width=0.48\linewidth]{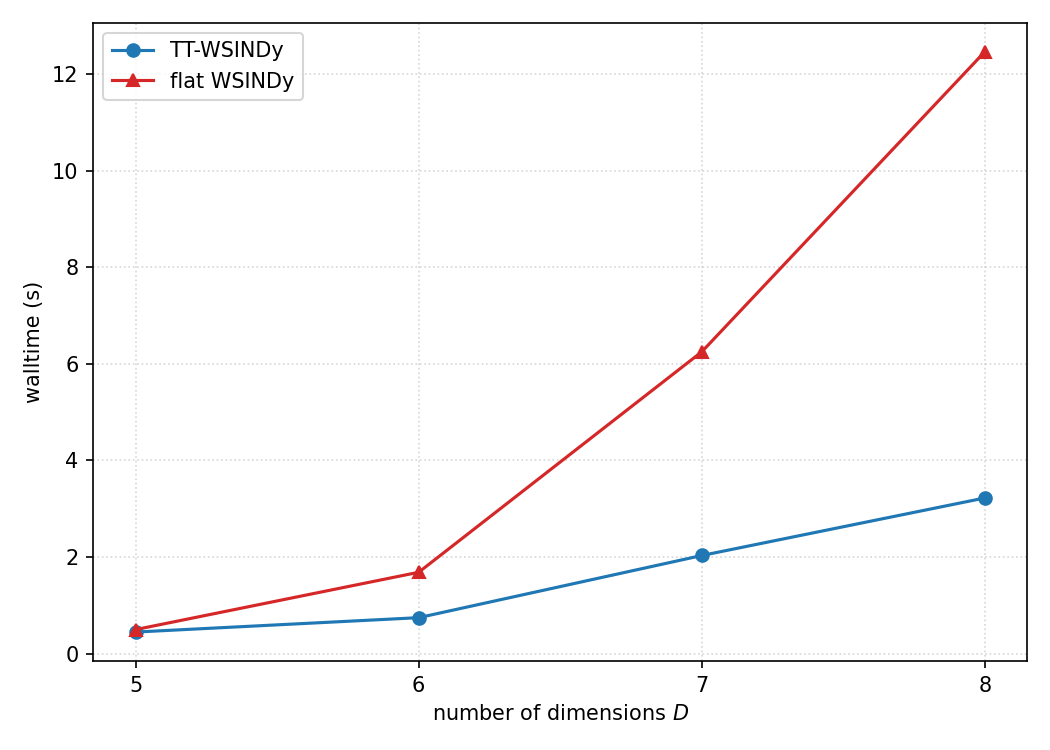}
    \includegraphics[width=0.48\linewidth]{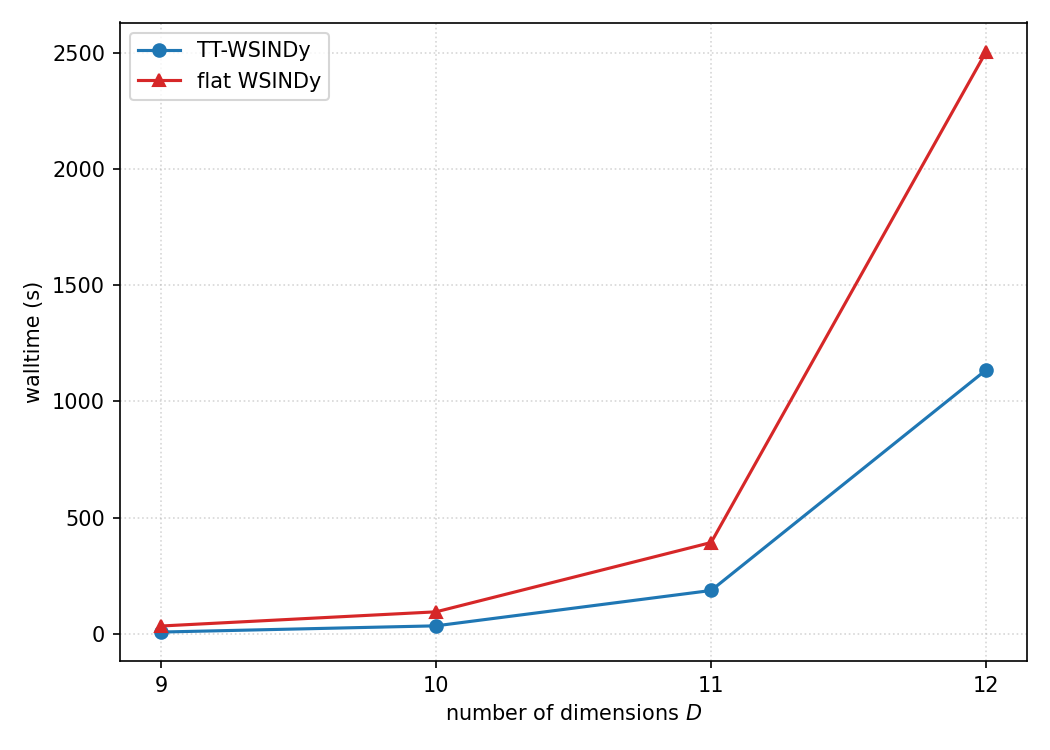}
    \caption{Walltime of flat WSINDy, and low-rank TT-WSINDy, averaged over 5 trials and plotted against the dimension of the Lorenz 96 system. For visibility, we plot the low-dimensional and high-dimensional experiments separately.}
    \label{fig:walltimevD}
\end{figure}

We observe the empirical computational gains in \cref{fig:walltimevD}. For $D=6,\dots,12$, TT-WSINDy outperformed flat WSINDy by a speedup factor of $2$ to $7$. $D=5$ was the changepoint at which the flat solver was faster; a coarse TT-MSTLS screening of the size $2^5$ search space did not yield computational benefits.
                           
\section{Conclusion}
We have extended the WSINDy algorithm using techniques from the MANDy method, achieving significant time and memory savings over a rich search space of candidate functions. We have shown that the strong-form problem can be translated into the weak form at little cost, and that the problem can be solved with existing WSINDy sparsification techniques without losing the time and memory complexity benefits of the tensor-train format. In particular, we showed that the TT-STLS stage has time and memory complexities that are polynomial in the dimension of the ODE rather than exponential. We have also made the algorithm tractable for datasets that are dense in the temporal axis, using the low-rank weak feature tensor construction.

We will close with future directions for this work. The first is to investigate more closely the statistical properties of the TT-MSTLS coarse screening stage. We would like the algorithm to avoid false negatives, so that the fine WSINDy pass reliably has the full library of true terms.

Secondly, we would like to examine the settings in which one feature tensor construction is applicable over the other (i.e., the original construction vs. that given in \cref{alg:lowrank_featuretensor}). While the original construction demonstrably overcomes the curse of dimensionality, it requires that the number of time snapshots is at most on same order of magnitude as the search space size $J^D$. The low-rank construction makes the MANDy regression applicable to high-data settings, but its time and memory complexities become dependent on the correlation structure of the data. A deeper understanding of this tradeoff can increase robustness of the tensor-based approach.

We also aim to extend this work to PDEs, in which dimensionality comes from the scale of the incoming data, in addition to the potentially large search space. Recent work has explored a similar candidate library in the setting of strong-form PDEs, taking the dictionary to be products of partial derivatives of basis functions \cite{He.Kang2026}. Additionally, a first pass at adapting MANDy to PDEs has been presented in \cite{LinLuZhang2024JComputNonlinearDync}. We would like to take a more detailed look at the time and memory complexities of this approach, and see whether it can be improved while adapting it to our own algorithmic techniques. In particular, we believe that the low-rank feature tensor approach can be effective in lessening the computational load that a large spatiotemporal dataset entails.
\appendix\section{MANDy\label{sec:MANDy}}
TT-WSINDy is built on top of the \textit{Multidimensional Approximation of Nonlinear Dynamics (MANDy)} method. In this paper, we also refer to this as the \textit{strong form} of TT-WSINDy. We describe it here.
\subsection{MANDy problem construction}
Given feature vectors and feature cores as in \cref{sec:WSINDyTconstruction}, we define the \textit{strong feature tensor} as
\begin{equation}
    \@\Theta(\!X) := \text{TT} \big\llbracket \@\Theta^{(1)},\dots, \@\Theta^{(D)}, \!I \big\rrbracket \in \R^{J \tct J \times M}
\end{equation}
where $\!I$ is the $M \times M$ identity matrix. Including $\!I$ as its own core is necessary for the MANDy construction to be equivalent to the matrix SINDy problem. The MANDy method then calculates the coefficient tensor as
\begin{equation}
    \widehat{\cW} = \dot{\!X} \tcont \@\Theta(\!X)^\dag 
\end{equation}
where $\@\Theta(\!X)^\dag$ is computed via a TT pseudoinverse.

\subsection{TT Pseudoinverse algorithm\label{sec:TT-PI}}
We give an overview of the procedure for computing the TT pseudoinverse \cref{eqn:TTPseudoinverse}, proposed in \cite{Klusetal2018Nonlinearity}. For more complete detail, see the original paper. 

The algorithm requires that the final TT core is \textit{right-orthonormal}, and the others are \textit{left-orthonormal}. We say a TT core $\cT^{(n)}$ is left-orthonormal if
\begin{equation*}
    \!L\big(\cT^{(n)}\big)^T\!L(\cT^{(n)}) = \!I \in \R^{R_n \times R_n}
\end{equation*}
and right-orthonormal if
\begin{equation*}
    \!R(\cT^{(n)})\!R\big(\cT^{(n)}\big)^T = \!I \in \R^{R_{n-1} \times R_{n-1}}
\end{equation*}
Algorithms to compute right- and left-orthonormal counterparts of a general TT core are given in \cite{Klusetal2018Nonlinearity}. In both, the computational effort is dominated by a singular value decomposition of the matricization of the core. Given these, the TT pseudoinverse algorithm is detailed in \cref{alg:TTPI} \footnote{In this algorithm, all left- and right-unfoldings are taken of order-3 tensors, so we  define the inverse unfoldings $\!L^{-1}, \!R^{-1}$ in the natural way, inverting the bijection $b_2$ to map a matrix to an order-3 tensor.}. Note that this algorithm ``splits'' $\cT$ at the index $N$. This choice is specific to the learning problem of this paper; any index in $[N+1]$ can be chosen. The more general version of this algorithm can be found in \cite{Klusetal2018Nonlinearity}.
\begin{algorithm}[htb]
\caption{\label{alg:TTPI} Pseudoinversion of tensor trains}
\begin{algorithmic}[1]
    \REQUIRE \quad Tensor cores $\cT^{(1)} \in \R^{1 \times I_1 \times R_1}$, $\cT^{(n)} \in \R^{R_{n-1} \times I_n \times R_n}$ for $n = 2,\dots, N$, $\cT^{(N+1)} \in \R^{R_N \times K \times 1}$
    \ENSURE Pseudoinverse $\cT^\dag \in \R^{K \times I_1 \tct I_N}$
    \FOR {$n = 1, \dots, N-1$}
        \STATE Compute $\!U^{(n)}\@\Sigma^{(n)}(\!V^{(n)})^T \leftarrow \text{SVD}(\!L(\cT^{(n)}))$
        \STATE Define $\widetilde{\cT}^{(n)} := \!L^{-1}(\!U^{(n)})$
        \STATE Recompute $\cT^{(n+1)} \leftarrow \@\Sigma^{(n)}(V^{(n)})^T \tcont \cT^{(n+1)}$
    \ENDFOR
    \STATE Compute $\!U^{(N+1)}\@\Sigma^{(N+1)}(\!V^{(N+1)})^T \leftarrow \text{SVD}(\cT^{(N+1)})$ 
    \STATE Recompute $\cT^{(N)} \leftarrow \!L(\cT^{(N)})\!U^{(N+1)}\Sigma^{(N+1)}$
    \STATE Compute $\!U^{(N)}\@\Sigma (\!V^{(N)})^T \leftarrow \text{SVD}(\cT^{(N)})$
    \STATE Define $\widetilde{\cT}^{(N)} := \!L^{-1}(\!U^{(N)})$
    \STATE Compute $\widetilde{\cT}^{(N+1)} \leftarrow (\!V^{(N)})^T \tcont \!R^{-1}((\!V^{(N+1)})^T)$
    \RETURN $\cT^\dag \leftarrow \text{TT}\big\llbracket \@\Sigma^{-1}\tcont \widetilde{\cT}^{(N+1)}, \widetilde{\cT}^{(1)},\dots,\widetilde{\cT}^{(N)}\big\rrbracket$
\end{algorithmic}
\end{algorithm}

The complexity of the algorithm is dominated by the cost of the constituent matrix SVDs. This gives the time complexity $\cO(NIR^3 + KR^2)$, presented in \cref{lem:TTPI-runtime}.

\subsection{Function-major feature tensor\label{sec:functionMajor}}
There is an alternate way to construct the feature tensor. In \cref{sec:WSINDyTconstruction}, we define the feature tensor over feature vectors that contain evaluations of each $f_j$ on a given equation. We may instead write
\begin{equation}
    \@\theta^{(j)}(t_m) = \begin{bmatrix}
        f_j(x_{1,m}) \\ \vdots \\ f_j(x_{D,m})
    \end{bmatrix} \in \R^D
\end{equation}
We refer to this as a \textit{function-major} feature vector, and \cref{eqn:dimmajorfeaturevector} as \textit{dimension-major}. The function-major feature tensor is then
\begin{equation*}
    \@\Theta(\!X,\@\varphi) \in \R^{\overbrace{D \tct D}^J \times M'}
\end{equation*}
While these representations correspond to different sizes of flat problems ($\cO(J^DM)$ versus $\cO(D^JM)$), both feature tensors have memory consumption in $\cO(JDM^2)$ and take $\cO(JDM^3)$ time to compute the pseudoinverse.

\section{Optimizations for TT-WSINDy}\label{sec:optimizations}
For completeness, we describe two additional optimizations in our TT-WSINDy implementation that were used to generate the walltime results in \cref{sec:walltime}.

First, we cache superfluous TT pseudoinverse solves across iterations. The pseudoinverse $\@\Theta^\dag$ has no dependence on the LHS data. Moreover, in each dimension $d$, we need only store the size $\cO(JD)$ set
\begin{equation}
    \text{tr}\big[ (\!W^{(d)}[j])^T\!D_L^{(d-1)}\!W^{(d)}[j] \!D_R^{(d+1)} \big], \quad d \in [D], j \in [J_d]
\end{equation}
for each support, rather than the entire pseudoinverse. So for a given support $S$, upon finding the associated set of traces, we can cache the low-memory set and reuse it when we need to threshold a different parameter $\lambda$ on $S$. This avoids unnecessary repeated calculations.

Second, we implemented a boolean \texttt{one\_pass} parameter that forces only a single coarse pass of TT-STLS. This is intended for ``medium-size'' search spaces -- those that are large enough that TT-WSINDy is computationally preferable, but small enough that a single iteration of coarsening is typically enough to hand off the problem to matrix MSTLS. For example, if we have $J = 5, D = 5$, the total search space for a single equation is $5^5 = 3125$. If, say, we find only 2 basis terms to be supported, the search space drops to $2^5 = 32$, at which point naively calling the TT pseudoinverse algorithm is unnecessary and can negate the computational savings of the algorithm.

For the Lorenz 96 system, the one-pass approach was sufficient. In more generality, this would take the form of a heuristic that reads the current shape parameters of the feature tensor and decides whether to perform TT sparsification or hand off the problem to the flat solver. We leave an investigation of such a heuristic to future work.

\begin{figure}
    \centering
    \includegraphics[width=\linewidth]{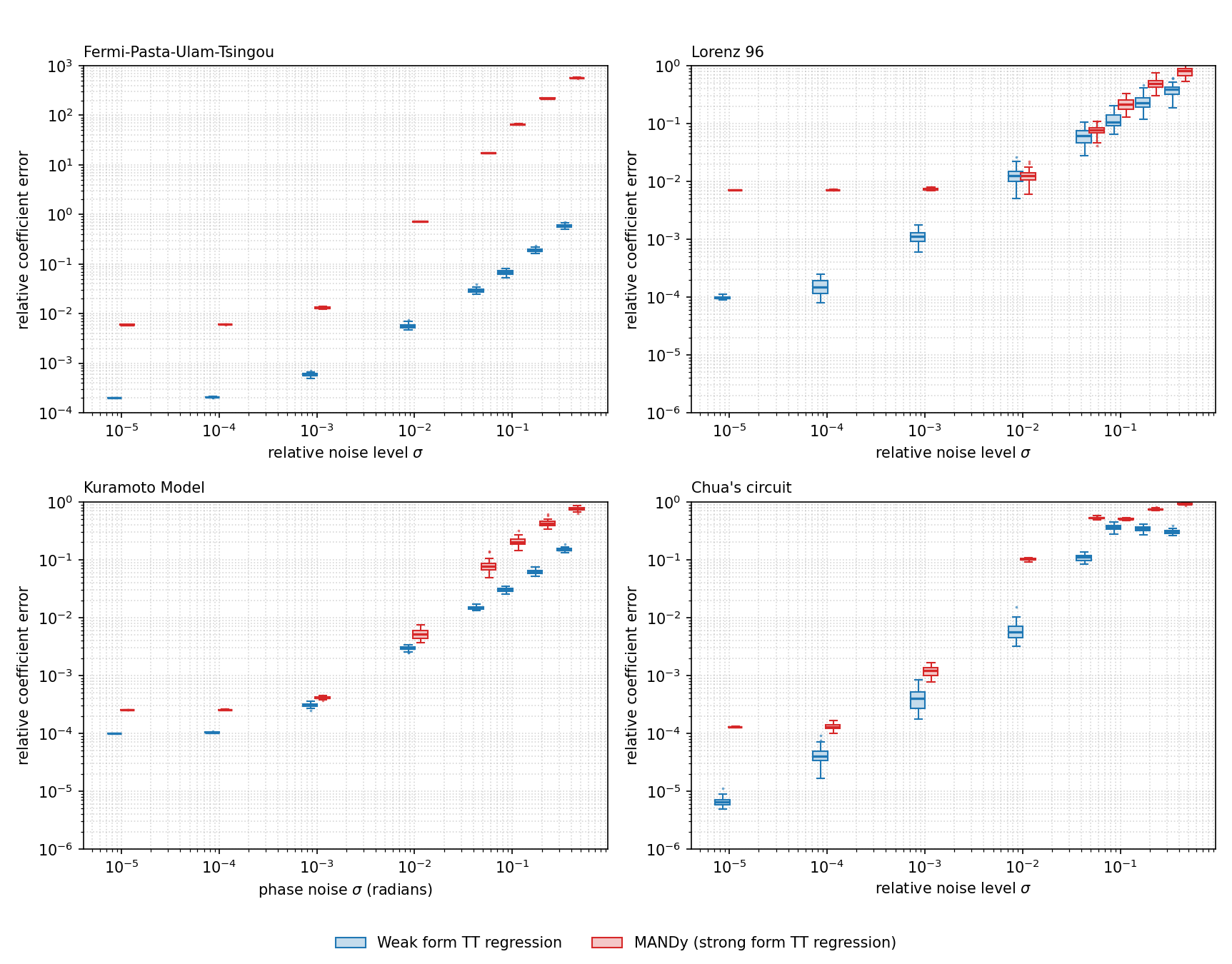}
    \caption{Box plot corresponding to the results in \cref{fig:weakvstrongform_all}.}
    \label{fig:weakvstrongform_boxplot}
\end{figure}

\section*{Acknowledgments}
This work was supported in part by the following grants: NIGMS R35GM149335, DOE DE-SC0023346, NSF DEB-2109774, and NIFA 2019-67014-29919. All authors stand by the results of this paper. Code used in this manuscript is publicly available at \href{https://github.com/whouser2001/TT-WSINDy.git}{https://github.com/whouser2001/TT-WSINDy.git}.

\bibliographystyle{siamplain}
\bibliography{bibs/tt-wsindy}

\end{document}